\documentclass{article}

\usepackage[preprint]{neurips_2026}

\usepackage[utf8]{inputenc} 
\usepackage[T1]{fontenc}    
\usepackage{hyperref}       
\usepackage{url}            
\usepackage{booktabs}       
\usepackage{amsfonts}       
\usepackage{nicefrac}       
\usepackage{microtype}      
\usepackage{xcolor}         
\usepackage{graphicx}
\usepackage{amsmath, amssymb, amsthm, amsfonts}
\usepackage{algorithmic}
\usepackage[ruled]{algorithm2e}
\usepackage{subcaption}
\usepackage{float}
\newcommand{\eos}{\mathrm{EoS}}
\newtheorem{proposition}{Proposition}
\newtheorem{assumption}{Assumption}
\newtheorem{definition}{Definition}
\newtheorem{lemma}{Lemma}

\newtheorem{theorem}{Theorem}
\newtheorem{claim}{Claim}
\usepackage{thmtools}
\usepackage{thm-restate}

\title{Fine-Tuning Dynamics of In-Context Factual Recall in Transformers}

\author{%
  Ruomin Huang \\
  Duke University\\
  \texttt{ruomin.huang@duke.edu} \\
  \And
  Eshaan Nichani \\
  Princeton University\\
  \texttt{eshnich@princeton.edu}
  \And
  Jason D. Lee\\
  UC Berkeley\\
  \texttt{jasondlee@berkeley.edu}
  \And
  Rong Ge\\
  Duke University\\
  \texttt{rongge@cs.duke.edu}
}

\begin{document}
\maketitle

\begin{abstract}

In-context learning \--- performing tasks based on examples given in the prompt \--- is an important capability that has emerged in large language models and has received significant attention in both theory and practice. Existing theoretical work often focuses on settings where the learning uses information purely from the prompt. However, many practical instances of in-context learning require the model to retrieve factual knowledge stored in the model's parameters, with the context serving to identify which knowledge is relevant. In this work, we study how in-context learning leverages factual knowledge recall. 
We formalize this behavior by introducing the \emph{in-context factual recall (IC-recall)} task, where a transformer is provided a context of (subject, answer) pairs generated from a hidden relation, along with a query subject, and must both infer this hidden relation and retrieve the corresponding answer. Factual knowledge is modeled by the transformer having access to a simple pre-constructed MLP associative memory storing (subject, relation, answer) triplets.
We analyze the supervised fine-tuning dynamics of a one-layer transformer on IC-recall data and prove that the model successfully performs IC-recall by converging to a particular pairwise attention pattern. 
This fine-tuning stage requires a very small number of samples \--- only polylogarithmic in the number of stored knowledge triplets. Experiments verify our theoretical predictions and show that the pairwise attention pattern emerges even when the MLP layer is pretrained instead of constructed.

\end{abstract}

\section{Introduction}
 Transformer-based large language models (LLMs) exhibit strong \emph{in-context learning} (ICL) abilities: they can use examples provided in the prompt to improve prediction on a new query \citep{DBLP:conf/nips/BrownMRSKDNSSAA20}. 
Considerable research has sought to understand the mechanisms underlying ICL, and most existing works focus on settings in which transformers rely only on the context to make predictions.  One line of work studies ICL over function classes \citep{DBLP:conf/nips/0001TLV22, DBLP:conf/icml/OswaldNRSMZV23, DBLP:journals/jmlr/ZhangFB24}, where transformers are trained from scratch on sequences of \((x, f(x))\) examples for a context-dependent function $f$, and learn to predict \(f(x_q)\) for a new query \(x_q\). 
Another line of work~\citep{DBLP:conf/icml/NichaniDL24,DBLP:conf/nips/Edelman0EMG24} studies how transformers learn to solve particular matching or copying tasks by converging to the induction head mechanism~\citep{DBLP:journals/corr/abs-2209-11895}.
Although these are interesting settings to study ICL from a theoretical perspective,
they do not capture an important behavior common in real-world LLM applications, where the model must combine contextual information with factual knowledge, or parametric knowledge~\citep{DBLP:conf/emnlp/PetroniRRLBWM19, DBLP:conf/emnlp/RobertsRS20,DBLP:journals/corr/abs-2410-08414}, stored in its weights. For example, given the prompt ``Albert Einstein $\to$ Germany, Isaac Newton $\to$ England, Marie Curie $\to$ ?'', the model must first use its stored knowledge to infer that the relation between subject and answer is nationality, then answer the query by retrieving the fact (again from stored knowledge) that Marie Curie was Polish. 

\begin{figure}[t]
    \centering
    \includegraphics[width=0.8\linewidth]{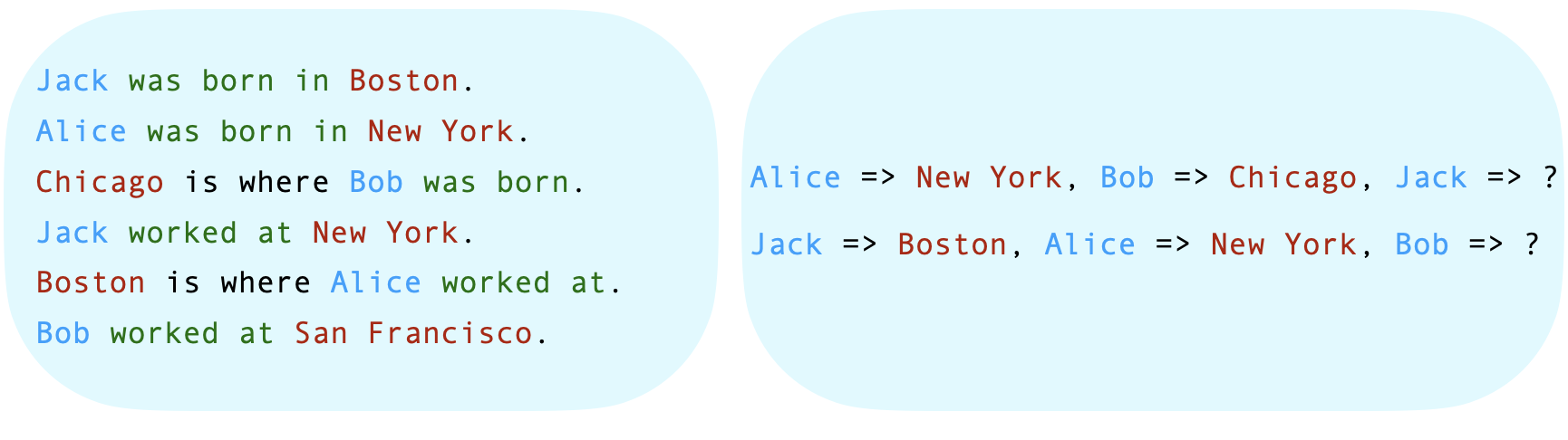}
    \caption{Left: factual knowledge corpus. Each sentence consists of three components: subject, relation and answer. We can put these components into the subject set $\mathcal{S}=\{\text{Jack, Alice, Bob}\}$, the relation set $\mathcal{R}=\{\text{was born in, worked at}\}$ and the answer set $\mathcal{A}=\{\text{Boston, New York, Chicago, San Francisco}\}$. We can view each relation $r\in\mathcal{R}$ as a mapping from $\mathcal{S}$ to $\mathcal{A}$.
    Right: IC-recall data requires the model to complete the sequence in a specific format: $(s_1,a_1,s_2,a_2,s_3, u_\eos)$. Here $u_\eos$ is the End-of-Sequence (EoS) token, $s_1,s_2\in\mathcal{S}$ and $a_1=r(s_1), a_2=r(s_2)$ for some underlying relation $r$.  }
    \label{fig:example}
\end{figure}

In this paper, we formalize this behavior through the \emph{in-context factual recall} (IC-recall) task.  Motivated by \cite{DBLP:conf/emnlp/PetroniRRLBWM19, DBLP:conf/icml/GhosalHR24, DBLP:conf/iclr/NichaniLB25}, we model factual knowledge as \((\text{subject}, \text{relation}, \text{answer})\) triplets. Each sequence in the IC-recall task consists of (subject, answer) pairs generated from the same relation, along with a query subject; the goal of the task is to output the answer corresponding to the query subject and hidden relation. Here, the context alone is insufficient to answer the query, as the model must also rely on the factual knowledge. See Figure \ref{fig:example} for an illustration of the task, and Section \ref{sec:preliminary} for a formal definition.


\subsection{Our Results}
 \label{sec:our_results}

We consider fine-tuning a model which already stores factual knowledge in its parameters on sequences from the IC-recall task. Specifically, the model is a transformer with a single self-attention head followed by an MLP layer. We construct an MLP layer that acts as an associative memory~\citep{DBLP:conf/iclr/NichaniLB25} storing (subject, relation, answer) triplets (see Section~\ref{sec:MLP}), and assumes that the MLP layer remains fixed throughout fine-tuning. 

We fine-tune the model to access this relevant knowledge in two steps: it first infers the underlying relation \(r\) from the context, and then uses \(r\) to predict the answer. This two-step structure naturally aligns with the \emph{chain-of-thought} (CoT) \citep{DBLP:conf/nips/Wei0SBIXCLZ22} paradigm, where an intermediate decoding step outputs \(r\) before producing the final answer. 



In Section~\ref{sec:main_results}, we analyze the supervised fine-tuning dynamics of this one-layer transformer on IC-recall data with two in-context examples. We prove that the transformer successfully performs IC-recall by converging to a {\em pairwise attention} pattern -- the transformer places equal attention weight to the tokens within each pair of in-context examples, but {\em different} attention weights across pairs. This identifies a mechanism for ICL which is qualitatively different from the induction head~\citep{DBLP:journals/corr/abs-2209-11895}.
We also show that the number of fine-tuning samples required to achieve high test accuracy on the IC-recall task grows polylogarithmically with the number of stored factual associations.

Empirically, we show in Section~\ref{sec:exp} that as predicted by theory, fine-tuning only requires a very small number of sequences (8 often already achieves high accuracy). We also show that similar pairwise attention patterns extend to the cases with more than two in-context examples, and when the MLP layer is pretrained instead of constructed.

\subsection{Related works}

\paragraph{Knowledge localization and associative memories.}
Prior work suggests that factual knowledge in language models can be stored in either MLP layers \citep{DBLP:conf/emnlp/GevaSBL21,DBLP:conf/nips/MengBAB22,DBLP:conf/acl/DaiDHSCW22}, though it may also be encoded in attention layers \citep{DBLP:conf/iclr/ChenC00025,DBLP:conf/cikm/WeiYWMZ0024}. Recent theory studies the ability of transformers to store factual knowledge via associative memories, showing that their storage capacity scales proportionally with model size \citep{DBLP:conf/iclr/Allen-ZhuL25,DBLP:conf/iclr/NichaniLB25,DBLP:conf/iclr/CabannesDB24}. \citet{ravfogel2026geometricfactualrecalltransformers} further propose a construction in which the required model size depends only logarithmically on the number of subjects. \citet{DBLP:conf/iclr/NichaniLB25} analyze the training dynamics of associative memory formation in one-layer linear transformers, and \citet{DBLP:conf/icml/CabannesSB24} study the training dynamics for linear layers.

\paragraph{Theoretical understanding of in-context learning.}
\citet{DBLP:conf/nips/0001TLV22} study the function classes that transformers can learn in context. \citet{DBLP:conf/icml/OswaldNRSMZV23} show that transformers can learn linear functions in context by implicitly simulating gradient descent, and \citet{DBLP:journals/jmlr/ZhangFB24} later prove convergence in this setting. \citet{DBLP:journals/corr/abs-2209-11895} identify induction heads as a mechanism underlying certain forms of in-context learning. \citet{DBLP:conf/icml/NichaniDL24, DBLP:conf/nips/Edelman0EMG24} analyze how induction heads emerge via gradient-based training on in-context Markov chains. \citet{DBLP:conf/nips/ChenSWY24} study how transformers learn a generalized induction head from in-context n-gram data. \citet{DBLP:conf/icml/Bu0HN0WS25} analyze factual-recall ICL under a data model that admits task-vector arithmetic, where answers are obtained by adding a retrieved task vector to the query representation. Among prior works, \citet{DBLP:journals/corr/abs-2603-20969} study the most similar data-generation setting, which they call the contextual recall task. Their work assumes disjoint answer sets across relations, providing an additional signal for ruling out irrelevant answers using the in-context examples. In contrast, our setting allows answers to be shared across different relations.

\paragraph{Fine-tuning dynamics.}
\citet{DBLP:conf/icml/MalladiWYCA23,DBLP:conf/iclr/RenS25,DBLP:conf/icml/ZengWL0C025} analyze fine-tuning dynamics through the neural tangent kernel (NTK) perspective. On related factual recall tasks, \citet{DBLP:conf/icml/GhosalHR24} study the fine-tuning dynamics of a self-attention layer in which the value matrix \(W_V\) stores facts.

\paragraph{Benefits of CoT.}
\citet{DBLP:journals/corr/abs-2603-09906} empirically discover that chain-of-thought (CoT) improves factual recall. \citet{DBLP:conf/iclr/MerrillS24,DBLP:conf/iclr/0001LZ024} show that CoT reasoning increases the expressive power of transformers. \citet{DBLP:conf/nips/AbbeBLSS24} show that CoT can expand the learnable function class of transformers. In addition, \citet{DBLP:conf/iclr/WenZLZ25,DBLP:conf/iclr/KimS25} show that CoT substantially improves the sample efficiency of transformers for \(k\)-parity. 

\section{Preliminaries}
\label{sec:preliminary}
\subsection{Transformer architecture}

\paragraph{Vocabulary.} Let the vocabulary be \(\mathcal{V} = \mathcal{S} \cup \mathcal{R} \cup \mathcal{A} \cup \{u_{\mathrm{EoS}}\}\), where \(\mathcal{S}\), \(\mathcal{R}\), and \(\mathcal{A}\) denote the sets of subjects, relations, and answers, respectively, and \(u_{\mathrm{EoS}}\) denotes the end-of-sequence (EoS) token. We assume that \(|\mathcal{S}| = |\mathcal{A}| =: n\), and that each \(r \in \mathcal{R}\) is a bijection from \(\mathcal{S}\) to \(\mathcal{A}\). For each pair \((s,r) \in \mathcal{S} \times \mathcal{R}\), let \(r(s) \in \mathcal{A}\) denote the associated answer. We also denote $\mathcal{R}(s,a)$ the set of relations that map $s$ to $a$. 

 \paragraph{Embedding.} All elements in vocabulary \(u \in \mathcal{V}\) are embedded into $\phi(u)\in\mathbb{R}^d$. For simplicity, we assume that all embeddings are unit vectors and mutually orthogonal, i.e., \(\langle \phi(u), \phi(u') \rangle = 0\) for all \(u \neq u'\). We also use one-hot positional encodings. Specifically, for the \(i\)-th input token, we concatenate its embedding with the one-hot vector \(e_i \in \mathbb{R}^{d_P}\), where \(d_P\) is the maximum input length. Thus, for an input sequence \(\widetilde Z=(u_1,\dots,u_k)\in\mathcal{V}^k\), we define the embedding matrix
\[
Z = E(\widetilde Z)
:=
\begin{bmatrix}
e_1 & e_2 & \cdots & e_k\\
\phi(u_1) & \phi(u_2) & \cdots & \phi(u_k)
\end{bmatrix}
\in \mathbb{R}^{(d+d_P)\times k}.
\]

\paragraph{Architecture.} We consider a one-layer transformer \(f(Z;\mathcal{W})\). Let \(z_{-1} \in \mathbb{R}^{d+d_P}\) denote the last token in the input embedding matrix \(Z\). The model output is given by
\[h = W^P\left(z_{-1} + Z\,\mathrm{softmax}\left(Z^\top W^{KQ} z_{-1}\right)\right),\]
and
\[f(Z;\mathcal{W}) = f_{\mathrm{MLP}}(h) = V^\top \sigma(Wh),\]
where \(W,V \in \mathbb{R}^{d_{\mathrm{MLP}}\times d}\), \(d_{\mathrm{MLP}}\) is the width of the MLP, $W^P = \begin{pmatrix}  0_{d\times d_P} & I_d\end{pmatrix}$
projects onto the token-embedding component only, and \(\mathcal{W}=(W^{KQ},W,V)\) are the set of parameters. We also assume that the activation in the MLP is the quadratic function \(\sigma(x)=x^2\) applied element-wise.


\paragraph{Construction of the MLP associative memory.} 

Transformers have been observed to memorize factual knowledge within the MLP layer. In practice, this memorization process typically occurs during the pretraining stage.
However, to enable theoretical analysis, we will assume the MLP layer is pre-constructed to store a given knowledge set. In Section~\ref{sec:exp}, we will observe empirically that pretrained MLPs behave similarly to our constructed MLP on the IC-recall task.

Consider a pretraining sequence \(\widetilde Z_{\mathrm{PT}} = (u_1,u_2)\)  that contains two randomly selected elements \(u_1,u_2\) from a valid fact triplet \((s,r,a)\). We say that the MLP stores this fact triplet if the argmax decoding
\[
\arg\max_{u \in \mathcal{V}} \phi(u)^\top f_{\mathrm{MLP}}(\phi(u_1)+\phi(u_2))
\]
recovers the remaining element of the triplet. We refer to an MLP layer which satisfies this property for all possible pretraining sequences as an \emph{MLP associative memory}. As shown later in Section~\ref{sec:MLP}, such an associative memory admits a simple and natural construction. In the following, we assume that the MLP associative memory used in the \emph{IC-recall task} is given by this construction. 

\subsection{IC-recall task}

We next formally define the IC-recall task.

\begin{definition}[IC-recall task]
   Given subject, relation, answer sets $\mathcal{S},\mathcal{R},\mathcal{A}$ and a one-layer transformer \(f(Z;\mathcal{W})\), whose MLP layer \(f_{\mathrm{MLP}}(\cdot)\) is preconstructed as an MLP associative memory,  the in-context factual recall  task consists of data generated as follows:
\begin{enumerate}
    \item Sample a relation \(r^* \in \mathcal{R}\) uniformly at random.
    \item Sample three distinct subjects \(s_1,s_2,s_3 \in \mathcal{S}\) uniformly at random, and define \(a_1 = r^*(s_1)\), \(a_2 = r^*(s_2)\), and \(a_3 = r^*(s_3)\). 
    \item Construct the IC-recall sequence
    \(
    \widetilde Z = (s_1,a_1,s_2,a_2,s_3,u_{\mathrm{EoS}}),
    \)
    with prediction target \(r^*, a_3\).
\end{enumerate}
We denote $P_{\mathrm{IC}}$ the distribution of such generated IC-recall sequence $\widetilde Z$.
\end{definition}

To make sure it is possible to correctly solve the IC-recall task, we assume that any two subject–answer pairs identifies at most one relation. 

\begin{assumption}\label{assum:identifiability}
For any \(s_1, s_2 \in \mathcal{S}\) and \(a_1, a_2 \in \mathcal{A}\), \(\bigl|\{\, r \in \mathcal{R} \mid r(s_1)=a_1,\ r(s_2)=a_2 \,\}\bigr| \le 1\).
\end{assumption}

The IC-recall task provides the true relation $r^*$ as an intermediate step of the reasoning, which makes the problem easier. Without this form of strong supervision, the partially fine-tuned 1-layer transformer cannot achieve accuracy greater than \(1/3\) (see Lemma \ref{lem:negative} in Appendix). Similar strong supervision was also considered in other works analyzing CoT \citep{DBLP:conf/iclr/WenZLZ25,DBLP:conf/iclr/KimS25}. We also remark that without such intermediate supervision, transformers struggle to learn compositional reasoning tasks~\citep{DBLP:conf/colt/WangNBDHLW25}. 

\subsection{Training loss for the IC-recall task}

We now define the loss used for fine-tuning. It is a chain-of-thought (CoT) objective in which the model decodes in two steps to produce the final answer. 
In the first step, the model predicts the relation token \(r\) associated with the input sequence; in the second step, it predicts the answer \(a_3\). For simplicity of analysis, we assume that the ``irrelevant'' logits are masked out during prediction. That is, in the first decoding step, only the logits corresponding to relation tokens are retained. Likewise, in the second decoding step, only the logits corresponding to answer tokens are retained. Therefore given an IC-recall sequence $\widetilde Z$, we can write the prediction of the first decoding step as \(p_1(\cdot ;\mathcal{W},\widetilde Z) := \mathrm{softmax}\left(A_rf(E(\widetilde Z);\mathcal{W})/T\right) \in \mathbb{R}^{|\mathcal{R}|}\) where \(A_r \in \mathbb{R}^{|\mathcal{R}|\times d}\) is the embedding matrix of all relations and $T>0$ is the prediction temperature. Similarly, conditioned on the first decoded relation being $r$, we write the prediction of the second decoding step as \(p_{2,r}(\cdot; \mathcal{W},\widetilde Z):=\mathrm{softmax}\left(A_af(E([\widetilde Z,r]);\mathcal{W})/T\right) \in \mathbb{R}^{|\mathcal{\mathcal{A}}|}\) where \(A_a \in \mathbb{R}^{|\mathcal{A}|\times d}\) is the embedding matrix of all answers. We sample a fine-tuning dataset $D$ from distribution $P_{\mathrm{IC}}$ and let

\begin{equation}
\label{eq:loss_1_finite_sample}
L_1(\mathcal{W},D)=\frac{1}{|D|}\sum_{\widetilde Z\in D}\ell_1(\mathcal{W}, \widetilde Z):=-\frac{1}{|D|}\sum_{\widetilde Z\in D}\log\left(p_1(r^*(\widetilde Z);\mathcal{W},\widetilde Z)\right)
\end{equation}

denote the cross-entropy loss for the first decoding step, where \(r^*(\widetilde Z)\) is the underlying relation for the sequence $\widetilde Z$. The second decoding step uses the cross-entropy loss conditioned on the correct relation $r^*(\widetilde Z)$.

\begin{equation}
\label{eq:loss_2_finite_sample}
L_2(\mathcal{W},D)=\frac{1}{|D|}\sum_{\widetilde Z\in D}\ell_2(\mathcal{W}, \widetilde Z):=-\frac{1}{|D|}\sum_{\widetilde Z\in D}\log\left(p_{2,r^*(\widetilde Z)}(a_3(\widetilde Z);\mathcal{W},\widetilde Z)\right)
\end{equation}

where \(a_3(\widetilde Z)\) is the correct answer for the sequence $\widetilde Z$. We use $r^*, a_3$ instead of $r^*(\widetilde Z), a_3(\widetilde Z)$ when the sequence is clear from context. The CoT loss is the sum of these two losses
\begin{equation}
    \label{eq:CoT_loss}
L(\mathcal{W},D)=L_1(\mathcal{W},D) + L_2(\mathcal{W},D).
\end{equation}


\paragraph{Partial fine-tuning.}
To simplify the theoretical analysis, we only fine-tune the position-position block \(W^{KQ}_{1:d_P,\,1:d_P}\) of the attention matrix, while keeping the MLP associative memory and the rest of the entries of $W_{KQ}$ fixed. There are only two vectors inside $W^{KQ}_{1:d_P,\,1:d_P}$ that are relevant for the two decoding steps. Let \(\theta:=W^{KQ}_{1:6,6} \in \mathbb{R}^6\) and \(\omega:=W^{KQ}_{1:7,7} \in \mathbb{R}^7\). Then $\mathrm{softmax}(\theta)$ and $\mathrm{softmax}(\omega)$ are the attention scores from the EoS token and the first decoded token respectively; the first decoding step depends solely on $\theta$, and the second depends solely on $\omega$. Therefore we can rewrite the loss as $L(\theta,\omega,D)$. In Section~\ref{sec:exp}, we show empirically that even without this restriction, the transformer finds the same solution as our theory predicts.

\section{MLP associative memory}
\label{sec:MLP}
In this section we provide a simple construction of the MLP layer with $d_{\mathrm{MLP}}=O(|\mathcal{S}|\cdot |\mathcal{A}|)$ that achieves $100\%$ accuracy as an MLP associative memory. This serves as the frozen MLP in our theoretical analysis.
\begin{lemma}
\label{lem:MLP_construction}
    There exists an MLP layer $f_{\mathrm{MLP}}(x)=V^\top \sigma(Wx)$ with width $d_{\mathrm{MLP}}=O(|\mathcal{S}|\cdot |\mathcal{A}|)$,  such that for any triplet $(s,r,a)$ where $r$ maps $s$ to $a$, we have  $a=\arg\max_{u\in\mathcal{V}} \phi(u)^\top f_{\mathrm{MLP}}(\phi(s)+\phi(r))$, $r=\arg\max_{u\in\mathcal{V}} \phi(u)^\top f_{\mathrm{MLP}}(\phi(a)+\phi(s))$ and  \(s=\arg\max_{u\in\mathcal{V}} \phi(u)^\top f_{\mathrm{MLP}}(\phi(a)+\phi(r)).\)
\end{lemma}

 \textbf{Proof sketch.}   Let $d_{\mathrm{MLP}}=3|\mathcal{S}|\cdot |\mathcal{A}|.$ Let $w_j$ and $v_j$ be the $j$-th row of $W$ and $V$ respectively. We assign indices to all $(s,a)$ pairs. For the $i$-th $(s,a)$ pair, we assign $3$ rows in $W$ and $V$ to store the relevant factual knowledge. Specifically, let $w_{3i-2}=\phi(s)+ \sum_{r\in\mathcal{R}(s,a)}\phi(r)$ and $v_{3i-2}=\phi(a)$; let $w_{3i-1}=\phi(a)+\phi(s)$ and $v_{3i-1}=\sum_{r\in\mathcal{R}(s,a)}\phi(r)$; let $w_{3i}=\phi(a)+\sum_{r\in\mathcal{R}(s,a)}\phi(r)$ and $v_{3i}=\phi(s)$ (we define $\sum_{r\in \mathcal{R}(s,a)}\phi(r)=0$ if $\mathcal{R}(s,a)$ is empty). It is easy to examine that such constructed MLP satisfies the properties in Lemma~\ref{lem:MLP_construction}; see the full proof in Appendix~\ref{sec:MLP_app}.



\textbf{Remark.}
If non-orthogonal embeddings are allowed, one can obtain a more parameter-efficient construction of MLP associative memory, with a total number of parameters nearly linear in \(|\mathcal{S}|\cdot |\mathcal{R}|\), as shown by \citet{DBLP:conf/iclr/NichaniLB25}. For simplicity in the subsequent dynamics analysis, we stick to the MLP construction in Lemma \ref{lem:MLP_construction}, which is based on orthogonal embeddings.

\section{Fine-tuning dynamics for IC-Recall}
\label{sec:main_results}
We show in this section that with the constructed MLP associative memory, transformers can achieve perfect accuracy on IC-recall data after fine-tuned by the perturbed gradient descent (PGD) in Algorithm \ref{alg:temp_perturbed_gd_1}. Specifically, we have the following Theorem:

\begin{theorem}
\label{thm:2nd_decoding_step}
Assuming that the IC-recall data satisfies Assumptions~\ref{assum:identifiability} and \ref{assum:sequence_distribution}, 
    there exist constants $C_1,C_2>0$ such that for any $T\leq C_1/\log n$, if we set $\eta_1=\Theta(T\sqrt{T}\log\frac{1}{T}), \eta_2=\Theta(T^2),t_1=1$ and $t_2\geq \frac{C_2}{T}\log \frac{1}{T}$ in Algorithm~\ref{alg:temp_perturbed_gd_1}, with probability at least $1-\delta$ over the sampled data $D$ and the random perturbation, the transformer fine-tuned on $\tilde O(\mathrm{poly}(\frac{1}{T}))$ \footnote{We use $\tilde O$ to hide $\mathrm{poly}\log\frac{1}{\delta}$ factors.} samples will have accuracy at least $0.99$ on all IC-recall sequences the for final predicted answers.
\end{theorem}

Note that Assumption~\ref{assum:identifiability} implies that it is possible to infer the hidden relation from only two in-context examples, and Assumption~\ref{assum:sequence_distribution} (introduced later in Section~\ref{sec:first-decoding-step-low-temp}) considers a particular difficult setting where the pairwise attention pattern is necessary to achieve good accuracy.
This suggests that $O(\mathrm{polylog}( n))$ samples are sufficient for the transformer to learn how to use the stored knowledge to perform IC-recall, far fewer than seeing every possible subjects/answers. 

\begin{algorithm}
\caption{Temperature-Scaled Perturbed Gradient Descent}
\label{alg:temp_perturbed_gd_1}
\begin{algorithmic}
\STATE \textbf{Input:} Decoding temperature $T > 0$, failure probability $\delta \in (0,1)$, learning rates $\eta_1,\eta_2$, number of iterations $t_1,t_2$, fine-tuning dataset $D$.

\STATE  \textbf{Initialization}: $(\theta,\omega) \gets 0$.
\STATE \textbf{Stage 1}:  Train $(\theta,\omega)$ using gradient descent with step size $\eta_1$ over loss $L(\theta,\omega,D; T)$ for $t_1$ iterations to obtain $(\tilde \theta,\tilde \omega)$.
\STATE \textbf{Stage 2}: Sample a perturbation
   $\xi \sim \mathrm{Uniform}\left(\mathbb{B}\left(0,\, \Theta\left(T^3 \log^{-2}(\frac{1}{\delta})\right)\right)\right)$, and set perturbed initialization $(\theta_0,\omega_0) \gets (\tilde\theta,\tilde \omega) + \xi$. Run gradient descent from $(\theta_0,\omega_0)$ with step size $\eta_2 $ over loss $L(\theta,\omega,D; T)$ for $t_2$ iterations.
\end{algorithmic}
\end{algorithm}


The first decoding step that recovers the underlying relation is harder to analyze. Algorithm \ref{alg:temp_perturbed_gd_1} is separated into two stages, where stage 1 is a single large gradient step and stage 2 is a perturbation followed by a few smaller gradient steps. We provide intuition on how the two stages work using a simple warm-up example with 3 subjects in Section~\ref{sec:3-subjects}. We then discuss the difficulty in extending such an analysis to the case when we do not see all subjects/answers in Section~\ref{sec:first-decoding-step-low-temp}, and show how this can be overcome by lowering the temperature $T$ in the decoding. Finally we complete the analysis of the second-step decoding in Section~\ref{sec:2nd_decoding}.



\subsection{Warm up: a three-subject instance}
\label{sec:3-subjects}
  We first consider the simple case where $|\mathcal{S}|=|\mathcal{A}|=3$ and $|\mathcal{R}|=6$. In this case, $\mathcal{R}$ consists all $6$ bijections from $\mathcal{S}$ to $\mathcal{A}$.  For an IC-recall sequence $(s_1,a_1,s_2,a_2,s_3,u_{\eos})$, we denote by $r_1$ the underlying correct relation and $r_2,\dots,r_6$ the remaining relations (see details in Table \ref{tab:relation_mapping}). The goal is to predict the correct relation $r_1$. We call the relation $r_2$ the confusing relation, and $r_3,...,r_6$ which assign $s_3$ to one of $a_1,a_2$ the mismatched relations. It is worth noting that, in this three-subject setting, the gradients satisfy
\(
\nabla_\theta \ell_1(\theta, Z_1)=\nabla_\theta \ell_1(\theta, Z_2)
\)
for any two IC-recall sequences \(Z_1\) and \(Z_2\). Consequently, the empirical loss is equivalent to the population loss, and it suffices to analyze the dynamics on a single fine-tuning sequence. 
As we will see later, even in this simple setting, transformers still learn a non-trivial solution and exhibit a two-stage fine-tuning dynamics.
  \begin{table}[ht]
\centering
\begin{tabular}{c|ccc|c}
\hline
 & \(s_1\) & \(s_2\) & \(s_3\) & type\\
\hline
\(r_1\) & \(a_1\) & \(a_2\) & \(a_3\) & correct \\
\(r_2\) & \(a_2\) & \(a_1\) & \(a_3\) & confusing\\
\(r_3\) & \(a_1\) & \(a_3\) & \(a_2\) & mismatched\\
\(r_4\) & \(a_2\) & \(a_3\) & \(a_1\) & mismatched\\
\(r_5\) & \(a_3\) & \(a_1\) & \(a_2\) & mismatched\\
\(r_6\) & \(a_3\) & \(a_2\) & \(a_1\) & mismatched\\
\hline
\end{tabular}
\caption{The six bijections \(r_1,\dots,r_6\) from \(\{s_1,s_2,s_3\}\) to \(\{a_1,a_2,a_3\}\).}
\label{tab:relation_mapping}
\end{table}

  Let \(v := \mathrm{softmax}(\theta)\) be the attention vector from the EoS token. We will show that in Algorithm \ref{alg:temp_perturbed_gd_1}, for a sufficiently small constant temperature $T$, if we set $\eta_1=\Theta(T^2), \eta_2=\Theta(T^2)$, $t_1=t_2=\infty$ so that Stage 1 and Stage 2 converge, then the fine-tuned transformer will converge to a global optimum where $v$ exhibits a pairwise pattern. Formally, we show the following Theorem.
  
  \begin{theorem}
\label{thm:pairwise_attention}
    There exists constant $T_{0}(\delta)>0$ that only depends on $\delta$, such that for any $T<T_{0}$, if we set $\eta_1 = \Theta(T^2), \eta_2 = \Theta(T^2), t_1=t_2=\infty$ in Algorithm \ref{alg:temp_perturbed_gd_1}, then with probability at least $1-\delta$ the transformer trained by Algorithm \ref{alg:temp_perturbed_gd_1} will converge to the global optimum $\theta^*$ for $L_1(\theta;T)$. At the convergence the attention scores $v^*=(a,a,\frac{1}{2}-a,\frac{1}{2}-a,0,0)$ with some $a\neq \frac{1}{4}.$
\end{theorem}
  
  The proof of Theorem \ref{thm:pairwise_attention} is based on a loss-landscape analysis and is deferred to Appendix~\ref{sec:app_normal_tem}. The convergence happens in two stages, and we provide a sketch of each stage below. 


\textbf{First stage: convergence to a saddle point.}
In the first stage, the attention scores \(v\) converge to the saddle point \(\tilde v = (1/4, 1/4, 1/4, 1/4, 0, 0)\). The transformer learns that $s_3$ and the $\mathrm{EoS}$ tokens are irrelevant in predicting the relation. The uniform attention over \(s_1,a_1,s_2,a_2\) arises from the \textit{symmetry} between the two pairs \((s_1,a_1)\) and \((s_2,a_2)\), together with the symmetric initialization. During this stage, \(v_5\) -- the attention score on the query subject \(s_3\) -- decreases, which makes the prediction logits of mismatched relations \(r_3,\dots,r_6\) significantly smaller than those of \(r_1\) and \(r_2\) (see Table~\ref{tab:logits} in the Appendix). By the end of this stage, the prediction of the transformer is uniform on \(r_1\) (the correct relation) and \(r_2\) (the confusing relation), but it cannot distinguish them, since \(v\) remains symmetric across the two subject-answer pairs.

\begin{figure}[t]
    \centering
    \begin{subfigure}{0.45\textwidth}
        \centering
\includegraphics[width=\linewidth]{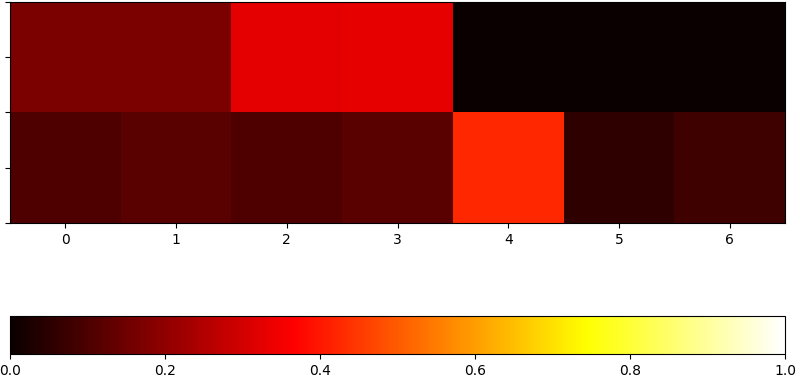}
        \caption{Fine-tuned on  $(s_1,a_1,s_2,a_2,s_3,u_{\eos})$}
    \end{subfigure}
    \hfill
    \begin{subfigure}{0.45\textwidth}
        \centering
        \includegraphics[width=\linewidth]{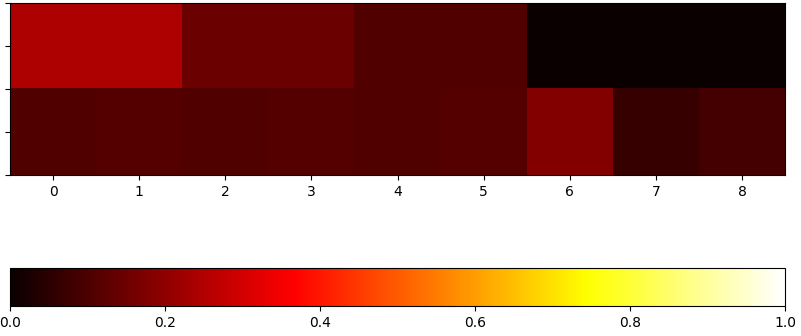}
        \caption{Fine-tuned on $(s_1,a_1,s_2,a_2,s_3,a_3,s_4,u_{\eos})$ }
    \end{subfigure}
    \caption{Pairwise attention at convergence for predicting the relation token in the first row. The first row is the attention scores from EoS token (first decoding step), and the second row is the attention scores from the decoded relation $r^*$ (second decoding step). We set $|\mathcal{S}|=|\mathcal{A}|=8$ and fix MLP as the construction in Lemma \ref{lem:MLP_construction} for both experiments. Left is the $2$-ICL-example input with $|\mathcal{R}|=64, T=0.05$ and right is $3$-ICL-example input with $|\mathcal{R}|=512, T=0.01$.}
    \label{fig:pairwise_att}
\end{figure}

\textbf{Second stage: escape the saddle.}
At the beginning of the second stage, a random perturbation is added. This perturbation breaks the symmetry between the two subject-answer pairs, so that after a few steps of gradient descent, the model escapes from the saddle point with high probability \citep{DBLP:conf/icml/Jin0NKJ17}. 
After escaping the saddle, the transformer begins to assign a higher prediction probability to \(r_1\) than to \(r_2\). Notably, a pairwise attention pattern emerges: the ratios \(v_1/v_2\) and \(v_3/v_4\) converge to \(1\), while the corresponding limiting attention scores are different. See Figure \ref{fig:pairwise_att} for an illustration of this pattern. Denote by \(l:=A_rf(Z;\theta)\) the logit vector of the prediction for the first decoding step, where $A_r$ is the embedding matrix for the the relations and the logit is a quadratic function \footnote{This is because the MLP layer uses the quadratic activation.} of the attention scores \(v\), as shown in Table~\ref{tab:logits} in the Appendix. The logit gap between \(r_1\) and \(r_2\) satisfies \(l(r_1)-l(r_2)\propto (v_1-v_3)(v_2-v_4)\). Hence, the transformer ensures that $v_1 - v_3$ and $v_2 - v_4$ have the same sign to make this gap positive. Moreover, by the AM--GM inequality, replacing both \(v_1,v_2\) by their mean \((v_1+v_2)/2\) and \(v_3,v_4\) by their mean \((v_3+v_4)/2\) can only increase this gap or leave it unchanged. Therefore $v_1$ and $v_2$ will be driven together, and likewise for $v_3$ and $v_4$, thus explaining the pairwise attention pattern. Finally, we remark that even though our analysis is for two in-context examples, the pairwise attention emerges empirically when there are more in-context examples (see Figure~\ref{fig:pairwise_att} (b)).


    

\subsection{Analysis on the first decoding step}
\label{sec:first-decoding-step-low-temp}
In this section, we consider the general case of \( |\mathcal{S}| = |\mathcal{A}|=n \ge 3 \).  Now, the empirical loss is no longer equivalent to the population loss. In particular, there may exist two subject--answer pairs that are not jointly matched by any relation, which breaks the symmetry among IC-recall sequences \(Z\). 
For an input sequence \(Z=(s_1,a_1,s_2,a_2,s_3,u_{\eos})\), we call a relation $r$ a \emph{2-matching relation} if $r$ maps two of the three subjects $\{s_1, s_2, s_3\}$ to the two answers $\{a_1, a_2\}$. We call \(r\) a \emph{confusing relation} (similar to $r_2$ in previous section) for \(Z\) if
\(
r(s_1)=a_2  \text{~and~}  r(s_2)=a_1
\). 
At the saddle point after stage 1, the model cannot distinguish the confusing relation and the correct relation. Similarly, we call a \(2\)-matching relation \(r\) a \emph{mismatched relation} for \(Z\) if \(r(s_3)=a_1\) or \(r(s_3)=a_2\). These correspond to the relations $r_3, r_4, r_5, r_6$ in the simple example in previous subsection. Finally, we call a sequence confusing/mismatched if it has at least one confusing/mismatched relation. Notice that a sequence might be both confusing and mismatched (in fact, all sequences in the simple example are both confusing and mismatched).

Now that we are working with the empirical loss, we will show that fine-tuning the transformer on $O(\mathrm{polylog} (n))$ random samples  can still achieve over $0.999$ test accuracy on predicting the underlying relation $r$ in the first decoding step. We denote $p_{\mathrm{conf}}$ the fraction of confusing sequences and $p_{\mathrm{mis}}$ the fraction of mismatched sequences in $D$. 
We assume that each of these two sequence types occupies at least a constant fraction of the IC-recall data distribution.

\begin{assumption}
\label{assum:sequence_distribution}
    There exists a constant $0<\zeta\le 1$, such that the probability a random IC-recall sequence is a confusing sequence is at least $\zeta$, and the probability it is a mismatched sequence is also at least $\zeta$.
\end{assumption}

Note that if no input sequences are confusing, the problem is easier in the sense that the saddle point that stage 1 converges to will also achieve 100\% accuracy. If there are no confusing sequences and no mismatched sequences, then even the initial solution would have 100\% accuracy. Therefore this assumption ensures we are considering the hardest case.

The loss-landscape analysis in Section~\ref{sec:3-subjects} relies on the symmetry of the loss in the pairs \(v_1,v_3\) and \(v_2,v_4\). This symmetry is difficult to generalize as the empirical loss is not symmetric. Another difficulty is that the attention pairing in Stage 2 relies on \(v_5\) converging to \(0\), which is requires many steps and is in fact unnecessary for achieving near-optimal performance. Therefore, in this section we instead analyze the low-temperature dynamics. 

With a lower temperature $T$, any improvement on the prediction logits will be amplified by $1/T$, which accelerates convergence. It therefore suffices to take only a single gradient descent step with a larger step size in Stage~1 and $O\left(\frac{1}{T}\log \frac{1}{T}\right)$ gradient descent steps in Stage~2. We show that this fine-tuning procedure also reaches a near-optimal solution. It is worth noting that, because the temperature is much lower, the near-saddle point reached after Stage~1 is closer to \((\frac{1}{6},\frac{1}{6},\frac{1}{6},\frac{1}{6},\frac{1}{6},\frac{1}{6})\) than to \((\frac{1}{4},\frac{1}{4},\frac{1}{4},\frac{1}{4},0,0)\), and the final solution does not exhibit the exact pairwise attention pattern (yet it still attends more to one pair compared to the other). Nevertheless, the resulting solution remains near-optimal in terms of accuracy. This property allows the training dynamics to be constrained in a local region near its initialization, and avoids difficult boundary cases when the attention to any of the positions gets near 0. Altogether, we show the following theorem, with proof deferred to Appendix~\ref{sec:proof_finite_sample}.

\begin{theorem}
\label{thm:finite_sample}
Assuming the IC-recall data satisfies Assumptions~\ref{assum:identifiability} and \ref{assum:sequence_distribution}, there exist constants $C_1,C_2>0$, such that for any $T\leq C_1/\log n$, if we set $\eta_1=\Theta(T\sqrt{T}\log\frac{1}{T}), \eta_2=\Theta(T^2), t_1=1$ and $t_2\geq \frac{C_2}{T}\log \frac{1}{T}$ in Algorithm \ref{alg:temp_perturbed_gd_1}, then with probability at least $1-\delta$ over samples $D$ and the random perturbation, transformers fine-tuned by Algorithm \ref{alg:temp_perturbed_gd_1} on $\tilde O\left(\mathrm{poly}\left(\frac{1}{T}\right)\right)$ samples will have accuracy at least $0.999$ for the first decoding step on all IC-recall sequences.
\end{theorem}

\subsection{Analysis on the second decoding step}
\label{sec:2nd_decoding}

The second decoding step is easier to analyze than the first decoding step. Intuitively, the transformer only needs to realize that attending to $s_3$, and combining it with the correct relation from the residual connection, is sufficient to produce the correct answer $a_3$. We show that this already happens at the end of Stage 1 in Lemma~\ref{lem:2nd_decoding}. The proof is deferred to Appendix~\ref{sec:proof_2nd_decoding}.  

\begin{restatable}{lemma}{secondDecoding}\label{lem:2nd_decoding}
Assume the IC-recall data satisfies Assumptions~\ref{assum:identifiability} and \ref{assum:sequence_distribution}.
    Set $T, \eta_1,\eta_2$ and the sample size $|D|$ same as in Theorem~\ref{thm:2nd_decoding_step} for Algorithm~\ref{alg:temp_perturbed_gd_1}, after Stage 1 and throughout Stage 2, the transformer has accuracy $1-o\left(\exp\left(-\frac{1}{\sqrt{T}}\right)\right)$ on all IC-recall sequences for the second decoding step.
\end{restatable}

 Theorem~\ref{thm:2nd_decoding_step} follows immediately by combining Theorem~\ref{thm:finite_sample} with Lemma~\ref{lem:2nd_decoding}.

\section{Experiments}
\label{sec:exp}
\textbf{Experiment setup.} Our model is a single-head attention layer followed by an MLP layer. We train the transformer using the Adam optimizer with learning rate $10^{-3}$. For all experiments, we fix the size $n$ of subjects $\mathcal{S}$ and answers $\mathcal{A}$, then we generate relations $\mathcal{R}$ as random bijections from $\mathcal{S}$ to $\mathcal{A}$. All subjects, answers, relations and EoS tokens use fixed random orthonormal embeddings. 

Figure~\ref{fig:acc_results} reports the test accuracies for the first decoding step on the IC-recall data. We can see that transformers can learn the IC-recall task with only $8$ samples in the experiment. The second decoding step consistently achieves perfect test accuracy across all hyperparameter settings and random seeds. 

\begin{figure}[t]
    \centering
    \includegraphics[width=0.5\linewidth]{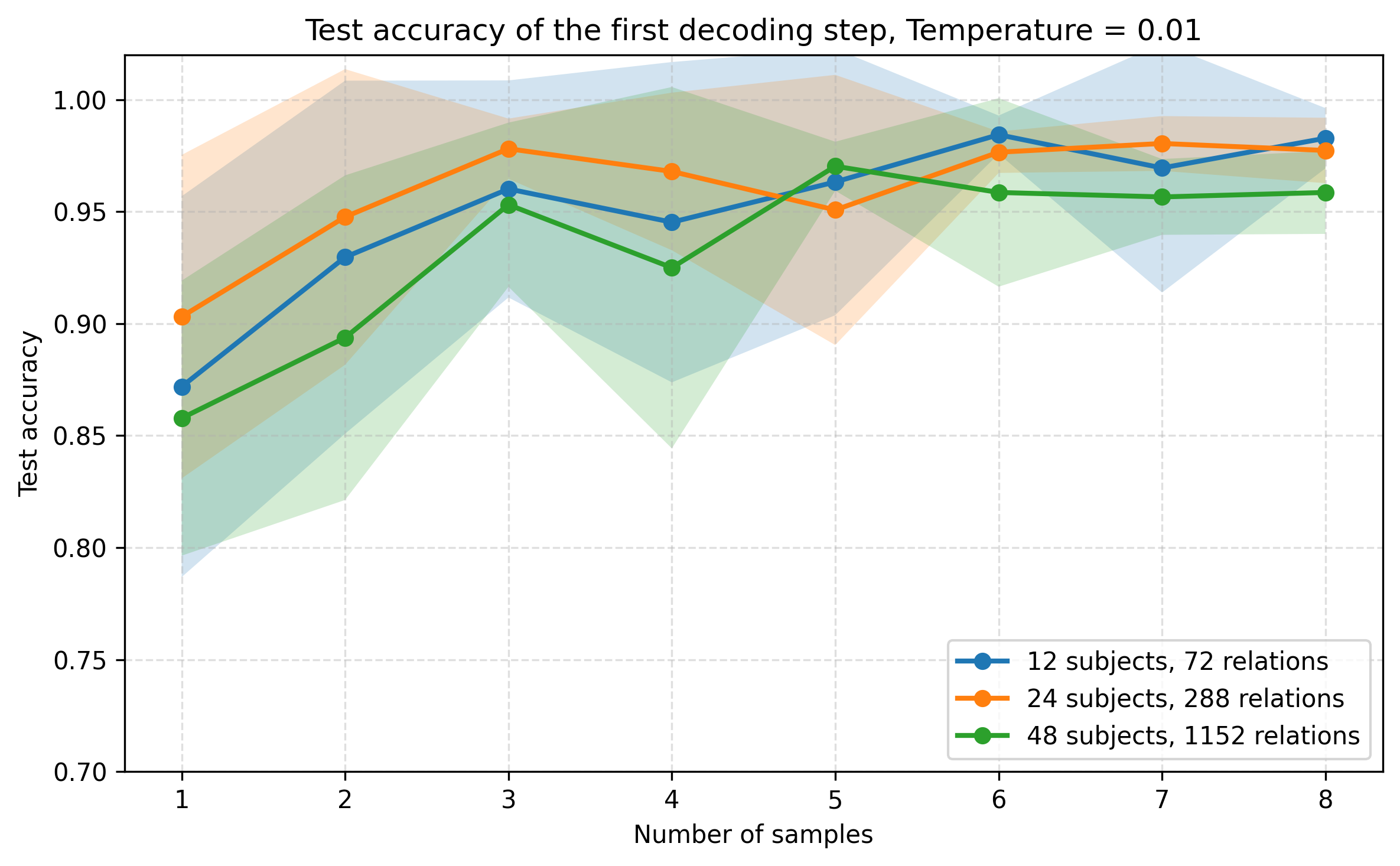}
    \caption{Test accuracies for fine-tuning on the first decoding step. The number of answers is equal to the number of subjects $|\mathcal{S}|$. The MLP is preconstructed and fixed during fine-tuning. We report the mean accuracies and standard deviations of $10$ random seeds.}
    \label{fig:acc_results}
\end{figure}

\textbf{Experiments with pretrained MLP.} We also conduct experiments with a pretrained MLP. Pretraining data \(\widetilde Z_{\mathrm{PT}} = \left(u_1,u_2\right)\) is generated as follows: 

\begin{enumerate}
    \item Sample a random subject $s$ from $\mathcal{S}$ and a random relation $r$ from $\mathcal{R}$. Let $a=r(s)$ and form a fact triplet $(s,r,a)$.
    \item Randomly choose two elements from this triplet and place them at positions \(u_1\) and \(u_2\). The task is to predict the remaining third element from the input sequence \(\widetilde Z_{\mathrm{PT}}\).
\end{enumerate}

During pretraining, the attention parameters $W^{KQ}$ and the MLP are both unfrozen. Then during fine-tuning, we freeze the MLP layer and  train $W^{KQ}$. Such fine-tuned transformer also exhibits the pairwise attention pattern. See Figure \ref{fig:pairwise_att_pretrained_MLP} for illustrations of converged attention scores. It is worth noting that attention weight for the second decoding step in Figure \ref{fig:pairwise_att_pretrained_MLP} is also assigned to the $\mathrm{EoS}$ token, which is different from Figure \ref{fig:pairwise_att}. As the $\mathrm{EoS}$ token does not contain any information, we believe the transformer uses it as a way to balance the coefficient between the attention to $s_3$ and the residual connection to $r^*$.

\begin{figure}[t]
    \centering
    \begin{subfigure}{0.45\textwidth}
        \centering
        \includegraphics[width=\linewidth]{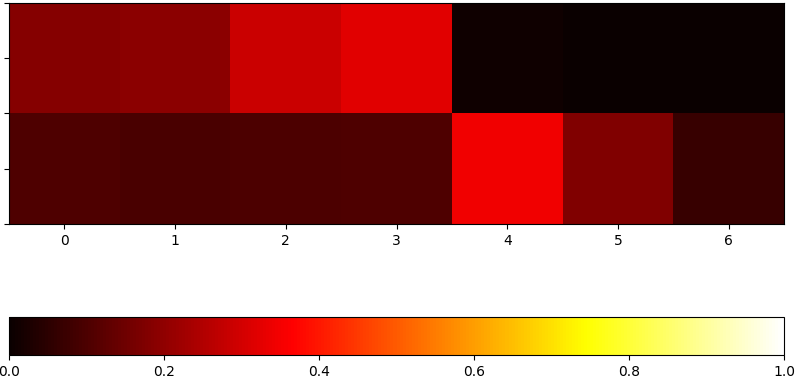}
        \caption{Fine-tuned on  $(s_1,a_1,s_2,a_2,s_3,u_{\eos})$}
    \end{subfigure}
    \hfill
    \begin{subfigure}{0.45\textwidth}
        \centering
        \includegraphics[width=\linewidth]{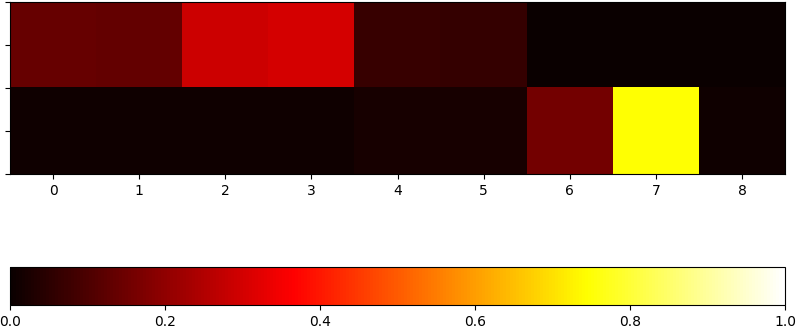}
        \caption{Fine-tuned on $(s_1,a_1,s_2,a_2,s_3,a_3,s_4,u_{\eos})$ }
    \end{subfigure}
    \caption{The pairwise attention pattern generalizes to pretraining the MLP. We consider pretraining the whole transformer and fine-tuning the whole attention layer. Left is the $2$-ICL-example input with $|\mathcal{S}|=|\mathcal{A}|=8,|\mathcal{R}|=64, T=0.05$ and right is the $3$-ICL-example input with $|\mathcal{S}|=|\mathcal{A}|=8, |\mathcal{R}|=512, T=0.01$.}
    \label{fig:pairwise_att_pretrained_MLP}
\end{figure}

\section{Conclusion}

We introduced the IC-recall task to study in-context learning when a transformer must infer a hidden relation from context and then use it to retrieve the correct answer from memory. We constructed an MLP associative memory for fact triplets and analyzed the fine-tuning dynamics of a one-layer transformer with the MLP frozen. We showed that fine-tuning is able to recover the hidden relation by converging to a pairwise attention pattern in the minimal setting and requiring only polylogarithmically many fine-tuning samples in the general setting. The main limitation of our work is that the analysis relies on a specially constructed MLP and a simplified architecture. Extending these results to general pretrained transformers and end-to-end training remains future work.
\section{Acknowledgements}
Rong Ge and Ruomin Huang acknowledge support of NSF Award DMS-2031849.
\bibliographystyle{unsrtnat}
\bibliography{main.bib}
\newpage
\appendix

\section{Construction of MLP associative memory}
\label{sec:MLP_app}
We show that there exists a construction of the MLP layer with $d_{\mathrm{MLP}}=O(|\mathcal{S}|\cdot |\mathcal{A}|)$ that achieves $100\%$ accuracy as an MLP associative memory.
\begin{lemma}
\label{lem:MLP_construction_app}
    There exists an MLP layer $f_{\mathrm{MLP}}(x)=V^\top \sigma(Wx)$ with width $d_{\mathrm{MLP}}=O(|\mathcal{S}|\cdot |\mathcal{A}|)$,  such that for any triplet $(s,r,a)$ where $r$ maps $s$ to $a$, we have  $a=\arg\max_{u\in\mathcal{V}} \phi(u)^\top f_{\mathrm{MLP}}(\phi(s)+\phi(r))$, $r=\arg\max_{u\in\mathcal{V}} \phi(u)^\top f_{\mathrm{MLP}}(\phi(a)+\phi(s))$ and  \(s=\arg\max_{u\in\mathcal{V}} \phi(u)^\top f_{\mathrm{MLP}}(\phi(a)+\phi(r)).\)
\end{lemma}

\begin{proof}
    Let $d_{\mathrm{MLP}}=3|\mathcal{S}|\cdot |\mathcal{A}|.$ Let $w_j$ and $v_j$ be the $j$-th row of $W$ and $V$ respectively. We assign indices to all $(s,a)$ pairs. For $i$-th $(s,a)$ pair, we assign $3$ rows in $W$ and $V$ to store the relevant factual knowledge. Specifically, let $w_{3i-2}=\phi(s)+ \sum_{r\in\mathcal{R}(s,a)}\phi(r)$ and $v_{3i-2}=\phi(a)$; let $w_{3i-1}=\phi(a)+\phi(s)$ and $v_{3i-1}=\sum_{r\in\mathcal{R}(s,a)}\phi(r)$; let $w_{3i}=\phi(a)+\sum_{r\in\mathcal{R}(s,a)}\phi(r)$ and $v_{3i}=\phi(s)$. Here we define $\sum_{r\in \mathcal{R}(s,a)}\phi(r)=0$ if $\mathcal{R}(s,a)$ is empty. Fix a $(s,a)$ pair such that $\mathcal{R}(s,a)$ is non-empty and we examine the properties of MLP associative memory as follows.

For any $r\in\mathcal{R}(s,a)$, since $r$ is an injection, we know $s$ is the only subject that $r$ maps to $a$. Therefore we can examine that $\phi(a)^\top f_{\mathrm{MLP}}(\phi(s)+\phi(r))=4$. For any $a'\neq a$, we know $a'$ can only be mapped from $s$ through relations that are not $r$. Therefore we have $\phi(a')^\top f_{\mathrm{MLP}}(\phi(s)+\phi(r))=1$. For any subject $s'$, we have $\phi(s')^\top f_{\mathrm{MLP}}(\phi(s)+\phi(r))=1$. For any relation $r'$, we have $\phi(r')^\top f_{\mathrm{MLP}}(\phi(s)+\phi(r))=1$. Therefore the constructed MLP can correctly predict $a$ from $\phi(s)+\phi(r)$ using argmax decoding. 

For any $r\in\mathcal{R}(s,a)$, we know $a$ is the only answer that $s$ is mapped to through $r$. Therefore we have $\phi(s)^\top f_{\mathrm{MLP}}(\phi(a)+\phi(r))=4$. For any $s'\neq s$, we know $a'$ can only be mapped from $s$ through relations that are not $r$. Therefore we have $\phi(s')^\top f_{\mathrm{MLP}}(\phi(a)+\phi(r))=1$. For any answer $a'$, we have $\phi(a')^\top f_{\mathrm{MLP}}(\phi(a)+\phi(r))=1$ if the preimage of $a'$ under $r$ exists; otherwise $\phi(a')^\top f_{\mathrm{MLP}}(\phi(a)+\phi(r))=0$. For any relation $r'$, we have $\phi(r')^\top f_{\mathrm{MLP}}(\phi(a)+\phi(r))=1$ if the preimage of $a$ under $r'$ exists; otherwise $\phi(r')^\top f_{\mathrm{MLP}}(\phi(a)+\phi(r))=0$. Therefore the constructed MLP can correctly predict $s$ from $\phi(a)+\phi(r)$ using argmax decoding.

For any $r\in\mathcal{R}(s,a)$, we know $\phi(r)^\top f_{\mathrm{MLP}}(\phi(s)+\phi(a))=4$. For any $r'\notin \mathcal{R}(s,a)$, we know $r'$ maps $s$ to a different answer $a'$ and the preimage of $a$ is a different subject $s'$ (if any). Therefore we have $\phi(r')^\top f_{\mathrm{MLP}}(\phi(s)+\phi(a))\leq 2.$ For any answer $a'$, we have $\phi(a')^\top f_{\mathrm{MLP}}(\phi(s)+\phi(a))=1$. For any subject $s'$, we have $\phi(s')^\top f_{\mathrm{MLP}}(\phi(s)+\phi(a))=1$. Therefore the constructed MLP can correctly predict a relation in $\mathcal{R}(s,a)$ from $\phi(s)+\phi(a)$ using argmax decoding. 
\end{proof}

\section{Proof to the normal temperature convergence for $n=3$}
\label{sec:app_normal_tem}
We first show that under a normal temperature $T$, with infinite amount of time, transformer trained by PGD can converge to the global optimal solution. For the first decoding step training, we will show Algorithm \ref{alg:temp_perturbed_gd_1} converges to the global optimal solution $\theta^*$ with small enough temperature $T$. At the convergence, the attention scores $v^*=\mathrm{softmax}(\theta^*)=(a,a,\frac{1}{2}-a,\frac{1}{2}-a,0,0)$ with some $a\neq \frac{1}{4}.$ We write out the logits w.r.t. the attention scores $v$ in Table \ref{tab:logits}. We can immediately obtain the following Lemma, which states that the partially fine-tuned transformer cannot predict the answer with one decoding step. 
\begin{lemma}
\label{lem:negative}
    Assume $|\mathcal{S}|=|\mathcal{A}|=3$ and $\mathcal{R}$ consists of all $6$ bijections from $\mathcal{S}$ to $\mathcal{A}$. Asumme the MLP layer $f_{\mathrm{MLP}}$ of the transformer $f(\cdot;W^{KQ})$ is fixed to be the construction in Lemma \ref{lem:MLP_construction_app}. If we use the first decoding step prediction to predict from answers $a_1,a_2,a_3$ without the intermediate relation, then the prediction accuracy cannot exceed $1/3$ for any input IC-recall data $\widetilde Z=(s_1,a_1,s_2,a_2,s_3,u_{\eos})$ and any $W^{KQ}$. 
\end{lemma}
\begin{proof}
    By the logits in Table \ref{tab:logits}, we can see that the logits $l(a_1)=l(a_2)=l(a_3)$ is always true for the first decoding step. Therefore we have $p_1(a_1;W^{KQ},\widetilde Z)=p_1(a_2;W^{KQ},\widetilde Z)=p_1(a_3;W^{KQ},\widetilde Z)$ for any IC-recall data $\widetilde Z=(s_1,a_1,s_2,a_2,s_3,u_{\eos})$ and any $W^{KQ}$.  The Lemma follows immediately.
\end{proof}
Before the proof of Theorem \ref{thm:pairwise_attention}, we first give the smoothness constant of the loss $L_1(\theta,T).$ Importantly through the gradient Lipschitz constant, we know that the learning rate $\eta_1$ and $\eta_2$ chosen in Theorem \ref{thm:pairwise_attention} can be small enough so that the descent lemma holds. The proof is deferred to Appendix \ref{sec:helper_lemmas}.
\begin{proposition}
\label{prop:smoothness}
    The loss $L_1(\theta;T)$ has gradient Lipschitz constant $\Theta(1/T^2)$ and Hessian Lipschitz constant $\Theta(1/T^3).$
\end{proposition}

Now we calculate the gradients. Let $l\in\mathbb{R}^6$ be the vector of prediction logits for relations. Let the prediction $p_1=\mathrm{softmax}(l)\in\mathbb{R}^6$. The loss $\ell_1(v)=-\log (p_1(r_1))$.  Therefore the gradient w.r.t. the logit $l$ is 
\[\frac{\partial \ell_1}{\partial l}=\frac{1}{T}(p_1-e_1).\]

The gradients w.r.t. $v$ are
\[\frac{\partial \ell_1}{\partial v_i}=\sum_{j=1}^6 \frac{\partial \ell_1}{\partial l_j}\frac{\partial l_j}{\partial v_i}.\]

\begin{table}[ht]
\centering
\begin{tabular}{|c|l|}
\hline
\(l(s_1)\) & \(v_2^2 + v_4^2\) \\ \hline
\(l(s_2)\) & \(v_2^2 + v_4^2\) \\ \hline
\(l(s_3)\) & \(v_2^2 + v_4^2\) \\ \hline
\(l(a_1)\) & \(v_1^2 + v_3^2 + v_5^2\) \\ \hline
\(l(a_2)\) & \(v_1^2 + v_3^2 + v_5^2\) \\ \hline
\(l(a_3)\) & \(v_1^2 + v_3^2 + v_5^2\) \\ \hline
\(l(r_1)\) & \((v_1+v_2)^2 + (v_3+v_4)^2 + v_5^2\) \\ \hline
\(l(r_2)\) & \((v_1+v_4)^2 + (v_2+v_3)^2 + v_5^2\) \\ \hline
\(l(r_3)\) & \((v_5+v_4)^2 + (v_1+v_2)^2 + v_3^2\) \\ \hline
\(l(r_4)\) & \((v_5+v_2)^2 + (v_1+v_4)^2 + v_3^2\) \\ \hline
\(l(r_5)\) & \((v_5+v_4)^2 + (v_3+v_2)^2 + v_1^2\) \\ \hline
\(l(r_6)\) & \((v_5+v_2)^2 + (v_3+v_4)^2 + v_1^2\) \\ \hline
\end{tabular}
\caption{The prediction logits (EoS excluded) of the input \((s_1,a_1,s_2,a_2,s_3,u_{\eos})\).}
\label{tab:logits}
\end{table}

Plugging in the logits in Table \ref{tab:logits} and noting that it is equivalent to analyze the gradient of a single sequence, we obtain
\begin{equation}
\label{eq:grad_v1}
    \frac{\partial L_1}{\partial v_1}=\frac{2}{T}\left(\left(p_1\left(r_1\right)+p_1\left(r_3\right)\right) v_2+\left(p_1\left(r_2\right)+p_1\left(r_4\right)\right) v_4-v_2\right)
\end{equation}

\begin{equation}
    \frac{\partial L_1}{\partial v_2}=\frac{2}{T}\left(\left(p_1\left(r_1\right)+p_1\left(r_3\right)\right) v_1+\left(p_1\left(r_2\right)+p_1\left(r_5\right)\right) v_3+\left(p_1\left(r_4\right)+p_1\left(r_6\right)\right) v_5-v_1\right)
\end{equation}

\begin{equation}
\label{eq:grad_v3}
  \frac{\partial L_1}{\partial v_3}=\frac{2}{T}\left(\left(p_1\left(r_1\right)+p_1\left(r_6\right)\right) v_4+\left(p_1\left(r_2\right)+p_1\left(r_5\right)\right) v_2-v_4\right)  
\end{equation}

\begin{equation}
    \frac{\partial L_1}{\partial v_4}=\frac{2}{T}\left(\left(p_1\left(r_1\right)+p_1\left(r_6\right)\right) v_3+\left(p_1\left(r_2\right)+p_1\left(r_4\right)\right) v_1+\left(p_1\left(r_3\right)+p_1\left(r_5\right)\right) v_5-v_3\right)
\end{equation}

\begin{equation}
    \frac{\partial L_1}{\partial v_5}=\frac{2}{T}\left(\left(p_1\left(r_3\right)+p_1\left(r_5\right)\right) v_4+\left(p_1\left(r_4\right)+p_1\left(r_6\right)\right) v_2\right)
\end{equation}

\begin{equation}
    \frac{\partial L_1}{\partial v_6}=0.
\end{equation}

The gradient w.r.t. $\theta$ is
\begin{equation}
\label{eq:grad_theta}
    \frac{\partial L_1}{\partial \theta_i}=v_i\left(\frac{\partial L_1}{\partial v_i}-\sum_{j=1}^6 v_j\frac{\partial L_1}{\partial v_j}\right).
\end{equation}

\begin{theorem}
\label{thm:pairwise_attention_app}
    There exists constant $T_{0}(\delta)>0$ that only depends on $\delta$, such that for any $T<T_{0}$, if we set $\eta_1 = \Theta(T^2), \eta_2 = \Theta(T^2), t_1=t_2=\infty$ in Algorithm \ref{alg:temp_perturbed_gd_1}, then with probability at least $1-\delta$ the transformer trained by Algorithm \ref{alg:temp_perturbed_gd_1} will converge to the global optimum $\theta^*$ for $L_1(\theta;T)$. At the convergence the attention scores $v^*=(a,a,\frac{1}{2}-a,\frac{1}{2}-a,0,0)$ with some $a\neq \frac{1}{4}.$
\end{theorem}
Theorem~\ref{thm:pairwise_attention_app} shows that, if the model is trained to convergence through Stages~1 and~2, a pairwise attention pattern emerges. The behavior in each stage is characterized by Lemma~\ref{lem:phase1}, Lemma~\ref{lem:escape_saddle_1}, and Lemma~\ref{lem:pairing}. In Stage~1, the model converges to a saddle point.
\begin{lemma}
\label{lem:phase1}
    At the end of Stage 1, $\theta$ will converge to a saddle point $\tilde \theta$ with $\tilde v:=\mathcal{S}(\tilde \theta)=(\frac{1}{4},\frac{1}{4},\frac{1}{4},\frac{1}{4},0,0)$.
\end{lemma}
\begin{proof}
    Some observations: \begin{enumerate}
        \item $v_1$ and $v_3$ are symmetric; $v_2$ and $v_4$ are symmetric along the gradient flow. 
        \item Initially $\nabla_\theta L_1(0,T)\neq 0$ so the loss will strictly decrease from $L_1(0,T)=\log 6$. 
        \item $v_5<v_1$ for any $t>0$. This is because any $v$ with $v_5\geq v_1$, $v_1=v_3$ and $v_2=v_4$ yields a loss $L_1(v,T)\geq \log 6.$ Similar we have that $v_4=v_2>0$ for any time $t>0$,  otherwise all six logits have same value and loss $L_1=\log 6$, which is contradictory to observation 2. 
        \item $\frac{\partial L_1}{\partial v_i}<0$ for $i=1,2,3,4$ and $\frac{\partial L_1}{\partial v_5}>0, \frac{\partial L_1}{\partial v_6}=0$ for any time $t>0.$
        \item $p_3=p_4=p_5=p_6$ for any time $t.$
    \end{enumerate}

    As $t_1 \rightarrow \infty$, at the end of Stage 1, $\theta$ will converge to a critical point of the loss. We know that this critical point must satisfy $v_1 = v_3$ and $v_4 = v_5$ and moreover obtain loss $< \log 6$. We first check the critical point condition: for any $i$, either $v_i=0$ or $\frac{\partial L_1}{\partial v_i}=\sum_{j=1}^6 v_j\frac{\partial L_1}{\partial v_j}$. Combined with observation $4$  and $3$, we have that the critical point along the GD trajectory must satisfy $v_5=v_6=0$, $v_1=v_3>0$ and $v_2=v_4>0$. So we can simplify the gradients to obtain $\frac{\partial L_1}{\partial v_1}=-2(p_5+p_6)v_2$ and $\frac{\partial L_1}{\partial v_2}=-2(p_4+p_6)v_1$. Then observation $5$ combined with critical point condition implies that $v_1=v_2=1/4$. Therefore $\tilde v$ is the only limit point of this gradient flow on the unit simplex, which is a saddle point as $(1,1,-1,-1,0,0)/2$ is a descending direction and $(1,-1,-1,1,0,0)/2$ is an ascending direction.

    It remains to prove observations 1,2 and 5.
    \begin{proof}[Proof to observations 1, 2, 5]
    We prove observations 1 and 5 by induction. Assume that observations 1 and 5 are true at iteration $k$. Comparing (\ref{eq:grad_v1}) and (\ref{eq:grad_v3}) and plugging them into (\ref{eq:grad_theta}), we can see that the gradients $\frac{\partial L_1}{\partial \theta_1}=\frac{\partial L_1}{\partial \theta_3}$ at iteration $k$. Similarly we can see that $\frac{\partial L_1}{\partial \theta_2}=\frac{\partial L_1}{\partial \theta_4}$ at iteration $k$. Therefore after one step of GD update, at iteration $k+1$, we still have $v_1=v_3$ and $v_2=v_4$, which implies $p_1(r_3)=p_1(r_4)=p_1(r_5)=p_1(r_6)$ at iteration $k+1$ through the logits $l$ in Table \ref{tab:logits}. It is obvious that the induction hypothesis holds at initialization $\theta=0$.

    At initialization, $v_i=\frac{1}{6},\frac{\partial L_1}{\partial v_1}=\frac{\partial L_1}{\partial v_3}=-\frac{1}{9T},$ $\frac{\partial L_1}{\partial v_2}=\frac{\partial L_1}{\partial v_4}=\frac{\partial L_1}{\partial v_6}=0,$ $\frac{\partial L_1}{\partial v_5}=\frac{2}{9T}.$ Therefore $v^\top \nabla_v L_1=0$ and $\nabla_\theta L_1=\frac{1}{6}\nabla_v L_1\neq 0$, and Observation 2 follows.
    \end{proof}
\end{proof}

We then show that in the late stage of training we have $l(r_1)>\sum_{r}p(r)l(r),$ which indicates that $\theta$ has escaped from the saddle point $\tilde \theta$.
\begin{lemma}
\label{lem:escape_saddle_1}
    If $T<T_{\max}(\delta)$ where $T_{\max}(\delta)$ is some constant only depends on $\delta$, then with probability at least $1-\delta$, after the perturbation followed by $O(\mathrm{poly}(\frac{1}{T})\cdot\log \frac{1}{\delta})$  steps of gradient descent, we have the loss $L_1<\log 2$.
\end{lemma}
The proof idea is that as $T$ gets close to $0$, $L_1(\tilde \theta, T)$ will converge to $\log 2$ from above. The rate of error decreasing is roughly $\exp(-1/T)$. We can calculate to see $\lambda_{\min}(\nabla^2L_1(\tilde \theta,T))\lesssim -\frac{1}{T}$ and the Hessian Lipschitz constant for $L_1$ is $O(1/T^3).$ 
By Lemma 14 in \cite{DBLP:conf/icml/Jin0NKJ17}, the loss will improve by a polynomial amount $L_1(\theta_t)-L_1(\tilde\theta)\leq -\Omega(T^3).$ Hence there exists some constant $T_{\max}(\delta)$ such that for any $T<T_{\max}$ we have the improvement surpasses the gap from $\log2$. For the ease of reference, we restate Lemma 14 in \cite{DBLP:conf/icml/Jin0NKJ17} here. 

\begin{lemma}[restated from Lemma 14 in \cite{DBLP:conf/icml/Jin0NKJ17}]
    Suppose $f(x)$ is an $\alpha$-gradient Lipschitz and $\rho$-Hessian Lipschitz function where $x\in\mathbb{R}^d$. There exists a universal constant $c$, for any $\delta\in(0,\frac{d\alpha}{\gamma}]$, suppose $\tilde x$ satisfies 
    \[\lambda_{\min}(\nabla^2f(\tilde x))\leq -\gamma \text{ and } \|\nabla f(\tilde x)\|\leq \sqrt{\eta\alpha}\frac{\gamma^2}{\rho}\cdot \log^{-2}(\frac{d\alpha}{\delta\gamma})\]
     with some $\gamma>0$ and $\eta\leq c/\alpha$.
    
    Let $x_0=\tilde x+\xi$ where $\xi\sim \mathrm{Uniform}(\mathbb{B}(0,\sqrt{\eta\alpha}\frac{\gamma^2}{\rho\alpha}\log^{-2}(\frac{d\alpha}{\delta \gamma})))$ and let $x_k$ be the $k$-th iteration of gradient descent starting from $x_0$ with stepsize $\eta,$ then with probability at least $1-\delta$, for any $t\geq \frac{\log(\frac{d\alpha}{\delta\gamma})}{c\eta\gamma}$ it holds that
    \[f(x_t)-f(\tilde x)\leq -\eta\alpha\frac{\gamma^3}{\rho^2}\log^{-3}(\frac{d\alpha}{\delta\gamma}).\]
\end{lemma}

Now we are ready to prove Lemma \ref{lem:escape_saddle_1}.
\begin{proof}[Proof to Lemma \ref{lem:escape_saddle_1}]
    Let $\tilde \epsilon=\frac{1}{2}\cdot (1,1,-1,-1,0,0)$ be the potential descending direction of loss. Define $f_{\tilde \epsilon,\tilde \theta}(x):=L_1(\tilde \theta+x\cdot\tilde \epsilon)$ to be the loss along the direction of $\tilde \epsilon$ from $\tilde\theta$. Then we have $\lambda_{\min}(\nabla^2L_1(\tilde \theta))\leq f_{\tilde \epsilon,\tilde \theta}^{''}(0).$ We use $p_1(x)$ to denote $p_1(r_1)$ at parameter $\theta=\tilde \theta+x\cdot \tilde \epsilon$. So we can write the loss along the direction of $\tilde\epsilon$ from $\tilde\theta$ as $f_{\tilde \epsilon,\tilde \theta}(x)=-\log p_1(x)$ and hence
    we can calculate to see that 

    \begin{equation}
    \label{eq:hessian}
        f_{\tilde \epsilon,\tilde \theta}^{''}(0)=\frac{(p_1'(0))^2}{p_1^2(0)}-\frac{p_1''(0)}{p_1(0)}.
    \end{equation}

    We first expand

    \begin{equation}
        \label{eq:p1'}p_1'(0)=\sum_{i}\frac{\partial p_1}{\partial l_i}l_i'(0)=\sum_i\frac{\partial p_1}{\partial l_i}\nabla_vl_i(v(0))^\top v'(0)
    \end{equation}
    
    where $l_i(x)$ is the logit for $r_i$ at parameter $\theta=\tilde \theta+x\cdot \tilde \epsilon$ and $v(x)=\mathrm{softmax}(\tilde\theta+x\cdot\tilde\epsilon)$ are the attention scores. So we can write 
    \[v'(0)=J(v(0))\tilde\epsilon=\left(\mathrm{Diag}(v(0))-v(0)v^\top(0)\right)\tilde\epsilon\]
where $J(v(0))=\mathrm{Diag}(v(0))-v(0)v^\top(0)$ is the Jacobian of softmax. Therefore we obtain 
\begin{equation}
    \label{eq:v'}v_i'(0)=\tilde\epsilon_iv_i(0)-v_i(0)\cdot v(0)^\top \tilde\epsilon.
\end{equation}
 Noting that $v(0)^\top \tilde\epsilon=0$, we have \[v'(0)=\tilde\epsilon \odot v=\left(\frac{1}{8},\frac{1}{8},-\frac{1}{8},-\frac{1}{8},0,0\right)^\top.\]
Also by the logits in table \ref{tab:logits} we have \[\nabla_v l_i(v(0))=\left(1,1,1,1,0,0\right)^\top ~\text{for }i=1,2,\]
 \[\nabla_v l_3(v(0))=\left(1,1,\frac{1}{2},\frac{1}{2},0,0\right)^\top, \]

 \[\nabla_v l_4(v(0))=\left(1,\frac{1}{2},\frac{1}{2},1,0,0\right)^\top, \]

\[\nabla_v l_5(v(0))=\left(\frac{1}{2},1,1,\frac{1}{2},0,0\right)^\top \]
and 
\[\nabla_v l_6(v(0))=\left(\frac{1}{2},\frac{1}{2},1,1,0,0\right)^\top.\]

Therefore we obtain 
\begin{equation}
    \label{eq:l'}l'(0)=\sum_i\nabla_vl_i^\top v'=(0,0,\frac{1}{8},0,0,-\frac{1}{8})^\top.
\end{equation}

Also since $p_1=\mathrm{softmax}(l/T)_1$, we have that 

\begin{equation}
\label{eq:dp1/dl}
 \frac{\partial p_1}{\partial l_i}=\frac{1}{T}p_1(\delta_{i1}-p_i).   
\end{equation}

Plugging (\ref{eq:l'}) and (\ref{eq:dp1/dl}) into (\ref{eq:p1'}) we have $p_1'(0)=0.$ 
Therefore we can simplify (\ref{eq:hessian}) to obtain

 \(f_{\tilde \epsilon,\tilde \theta}^{''}(0)=-\frac{p_1''(0)}{p_1(0)}\) where 
 \[p_1''(0)=\underbrace{\sum_{i=1}^6\frac{\partial p_1}{\partial l_i}l_i''(0)}_{\text{term I}}+\underbrace{\sum_{i,j=1}^6\frac{\partial^2 p_1}{\partial l_i\partial l_j}l_i'(0)l'_j(0)}_{\text{term II}}.\]
We first calculate term I. Since $l'_i(0)=\nabla_vl_i(v(0))^\top v'(0)$, we can differentiate to obtain 
\[l''_i=v'^\top \nabla^2_v l_i(v)v'+l_i'(v)^\top v''.\]

Differentiating on (\ref{eq:v'}) yields 
\begin{align*}
   v_i''(0)&=\tilde\epsilon_iv_i'(0)-v_i'(0)v(0)^\top \tilde\epsilon-v_i(0)v'(0)^\top\tilde\epsilon\\
   &=v_i(0)\left(\tilde\epsilon_i^2-\sum_{j=1}^6v_j(0)\tilde\epsilon_j^2\right)\\
   &=0.
\end{align*}
Therefore we can simplify $l_i''$ to obtain
\begin{equation}
    \label{eq:l''}
    l_i''(0)=v'(0)^\top \nabla^2_v l_i(v(0))v'(0)
\end{equation}

where we have 
\[\nabla^2_v l_1(v(0))=\begin{pmatrix}
2 & 2 & 0 & 0 & 0 & 0 \\
2 & 2 & 0 & 0 & 0 & 0 \\
0 & 0 & 2 & 2 & 0 & 0 \\
0 & 0 & 2 & 2 & 0 & 0 \\
0 & 0 & 0 & 0 & 2 & 0 \\
0 & 0 & 0 & 0 & 0 & 0
\end{pmatrix},\]

\[\nabla^2_v l_2(v(0))=\begin{pmatrix}
2 & 0 & 0 & 2 & 0 & 0 \\
0 & 2 & 2 & 0 & 0 & 0 \\
0 & 2 & 2 & 0 & 0 & 0 \\
2 & 0 & 0 & 2 & 0 & 0 \\
0 & 0 & 0 & 0 & 2 & 0 \\
0 & 0 & 0 & 0 & 0 & 0
\end{pmatrix},\]

\[\nabla^2_v l_3(v(0))=\begin{pmatrix}
2 & 2 & 0 & 0 & 0 & 0 \\
2 & 2 & 0 & 0 & 0 & 0 \\
0 & 0 & 2 & 0 & 0 & 0 \\
0 & 0 & 0 & 2 & 2 & 0 \\
0 & 0 & 0 & 2 & 2 & 0 \\
0 & 0 & 0 & 0 & 0 & 0
\end{pmatrix},\]

\[\nabla^2_v l_4(v(0))=\begin{pmatrix}
2 & 0 & 0 & 2 & 0 & 0 \\
0 & 2 & 0 & 0 & 2 & 0 \\
0 & 0 & 2 & 0 & 0 & 0 \\
2 & 0 & 0 & 2 & 0 & 0 \\
0 & 2 & 0 & 0 & 2 & 0 \\
0 & 0 & 0 & 0 & 0 & 0
\end{pmatrix},\]

\[\nabla^2_v l_5(v(0))=\begin{pmatrix}
2 & 0 & 0 & 0 & 0 & 0 \\
0 & 2 & 2 & 0 & 0 & 0 \\
0 & 2 & 2 & 0 & 0 & 0 \\
0 & 0 & 0 & 2 & 2 & 0 \\
0 & 0 & 0 & 2 & 2 & 0 \\
0 & 0 & 0 & 0 & 0 & 0
\end{pmatrix},\]

and 
\[\nabla^2_v l_6(v(0))=\begin{pmatrix}
2 & 0 & 0 & 0 & 0 & 0 \\
0 & 2 & 0 & 0 & 2 & 0 \\
0 & 0 & 2 & 2 & 0 & 0 \\
0 & 0 & 2 & 2 & 0& 0 \\
0 & 2 & 0 & 0 & 2 & 0 \\
0 & 0 & 0 & 0 & 0 & 0
\end{pmatrix}.\]
 Combining them with $v'(0)=\left(\frac{1}{8},\frac{1}{8},-\frac{1}{8},-\frac{1}{8},\right)^\top$ and plugging into \ref{eq:l''}, we have 
 \begin{equation}
     \label{eq:l''-value}
     l''(0)=\left(\frac{1}{4},0,\frac{3}{16},\frac{1}{16},\frac{1}{16},\frac{3}{16}\right).
 \end{equation}

 At saddle point $x=0$ we have  $p_1(r_1)=p_1(r_2)=p_1(0)$ and $p_1(r_3)=p_1(r_4)=p_1(r_5)=p_1(r_6)=\frac{1-2p_1(0)}{4}$. Combining them with (\ref{eq:dp1/dl}), (\ref{eq:l''-value}) and plugging into term I, we obtain
\[\text{term I}=\frac{p_1(0)}{8T}.\]
 
     Then we work with term II. Recall $\text{term II}=\sum_{i,j=1}^6\frac{\partial^2 p_1}{\partial l_i\partial l_j}l_i'(0)l'_j(0)$ and $l'(0)=(0,0,\frac{1}{8},0,0,-\frac{1}{8})^\top$ in (\ref{eq:l'}). Hence we only need to consider $l'_3(0)$ and $l'_6(0)$ since the rest of $l'_i$ are zero. Specifically we have 
     \[\frac{\partial^2 p_1}{\partial l_3^2}=\frac{p_1(0)}{T^2}\left(2p_1(r_3)^2-p_1(r_3)\right), \frac{\partial^2 p_1}{\partial l_6^2}=\frac{p_1(0)}{T^2}\left(2p_1(r_6)^2-p_1(r_6)\right) \]
     and \[\frac{\partial^2 p_1}{\partial l_3\partial l_6}=\frac{2p_1(0)p_1(r_3)p_1(r_6)}{T^2}.\] Plugging them into term II, we obtain
     \(\text{term II}=\frac{-p_1(0)p_1(r_3)^2}{16T^2}.\) 
     Noting that $p_1(r_3)$ at the saddle is exponentially small,
    we obtain 
    \[f_{\tilde \epsilon,\tilde \theta}^{''}(0)=-\frac{1}{8T}+\frac{1}{16T^2\left(4+2\exp(\frac{3}{16T})\right)^2}\lesssim -\frac{1}{T}.\]
    So we have $\lambda_{\min}(\nabla^2L_1(\tilde \theta))\lesssim-\frac{1}{T}.$ 

    By proposition \ref{prop:smoothness}, we have the Hessian Lipschitz constant is $O(1/T^3).$ Plug the upper bound on $\lambda_{\min}(\nabla^2L_1(\tilde \theta))$ and the Hessian Lipschitz constants into Lemma 14 in \citet{DBLP:conf/icml/Jin0NKJ17}, we have the improvement of loss being $\Omega(T^3)$, which exceeds the gap between $L_1(\tilde \theta)$ and $\log 2$ if $T$ is small enough. 
    
\end{proof}
Adding perturbation $\xi$ to the saddle point $\tilde \theta$ would not change $v_5$ and $v_6$ since $\tilde \theta_5$ and $\tilde \theta_6$ are $-\infty$ at the convergence of Stage 1. So we need only to consider the perturbation on the first $4$ coordinates. It turns out that a pairwise attention pattern will emerge along the gradient descent dynamics after escaping the saddle point.

\begin{lemma}
\label{lem:pairing}
    Denote $\alpha$ the ratio $\frac{v_1}{v_2}$ and $\beta$ the ratio $\frac{v_3}{v_4}$. Define $\psi:=\max\{\alpha,\beta,\frac{1}{\alpha},\frac{1}{\beta}\}$. If $l(r_1)>\sum_{r}p_1(r)l(r)$ and $v_5=0$, then $\psi$ will be decreasing unless $\alpha=\beta=1$.
\end{lemma}
\begin{proof}[Proof sketch of Lemma \ref{lem:pairing}]
We have that $\frac{d\alpha}{dt}<0$ is equivalent to $\frac{\partial L_1}{\partial \theta_1}-\frac{\partial L_1}{\partial \theta_2}>0.$
    By calculations we have $\sum_j v_j \frac{\partial L_1}{\partial v_j} = \sum_r p_1(r)l(r)$, which implies that 

    \begin{align*}
        &T\left(\frac{\partial L_1}{\partial \theta_1}-\frac{\partial L_1}{\partial \theta_2}\right)\\=&(p_1(r_2)+p_1(r_4))v_1v_4-(p_1(r_2)+p_1(r_5))v_2v_3-(p_1(r_4)+p_1(r_6))v_2v_5+(v_1-v_2)\left(l(r_1)-\sum_{r}p_1(r)l(r)\right)
    \end{align*}
    
    and 
    \begin{align*}
        &T\left(\frac{\partial L_1}{\partial \theta_3}-\frac{\partial L_1}{\partial \theta_4}\right)\\=&(p_1(r_2)+p_1(r_5))v_2v_3-(p_1(r_2)+p_1(r_4))v_1v_4-(p_1(r_3)+p_1(r_5))v_4v_5+(v_3-v_4)\left(l(r_1)-\sum_{r}p_1(r)l(r)\right).
    \end{align*}
    Note that if $\alpha\geq \beta,$ then we have $v_1v_4\geq v_2v_3$ and $p_1(r_4)\geq p_1(r_5)$. Furthermore if $\alpha\geq 1$, combining it with $v_5=0$, we have $\frac{\partial L_1}{\partial \theta_1}-\frac{\partial L_1}{\partial \theta_2}\geq0$. It is worth noting that the equality   $\frac{\partial L_1}{\partial \theta_1}=\frac{\partial L_1}{\partial \theta_2}$ iff $\alpha=\beta=1$. Therefore we can deduce that $\frac{d\alpha}{dt}\leq 0$ from $\alpha\geq\beta$ and $\alpha\geq1$.The equality holds iff $\alpha=\beta=1$. 

    Now we consider 4 cases:
    \begin{enumerate}
        \item $\alpha\geq\beta\geq 1\geq 1/\beta\geq 1/\alpha.$
        In this case, we have $\alpha\geq\beta$ and $\alpha\geq 1$. Therefore  we have $\frac{d\alpha}{dt}\leq 0$ and the equality holds iff $\alpha=\beta=1$. 
        \item $\alpha\geq1/\beta\geq 1\geq\beta\geq 1/\alpha.$ In this case, we have $\alpha\geq\beta $ and $\alpha\geq 1$. Therefore we also have $\frac{d\alpha}{dt}\leq 0$ and the equality holds iff $\alpha=\beta=1$.
        \item $1/\beta \geq \alpha\geq 1\geq 1/\alpha\geq\beta.$ In this case, we have $\frac{\alpha}{\beta}=\frac{v_4}{v_3}\cdot\frac{v_1}{v_2}\geq 1$. Hence we have $v_1v_4\geq v_2v_3$ and thus $p_1(r_4)\geq p_1(r_5)$. We also know $v_4\geq v_3$ since $1/\beta\geq 1$. Plugging them into $T\left(\frac{\partial L_1}{\partial \theta_3}-\frac{\partial L_1}{\partial \theta_4}\right)$ we have $\frac{\partial L_1}{\partial \theta_3}\leq \frac{\partial L_1}{\partial\theta_4}$, which is equivalent to $\frac{d(1/\beta)}{dt}\leq 0$. It is straightforward to verify that the equality holds iff $v_3=v_4$ and $v_1=v_2$, which is equivalent to $\alpha=\beta=1$.  
        \item $1/\beta\geq 1/\alpha\geq 1\geq \alpha \geq \beta.$ In this case, by exact same argument in case 3, we have $\frac{d(1/\beta)}{dt}\leq 0$ and the equality holds iff $\alpha=\beta=1$.
    \end{enumerate}
     The rest $4$ cases are symmetric to these $4$ cases (by swapping $\alpha$ and $\beta$). Therefore we have $\frac{d \psi}{dt}<0$ as long as $\psi>1$.

\end{proof}

Now we are ready to prove Theorem \ref{thm:pairwise_attention_app}. Actually it only remains to show the convergent point is the global optimum.
\begin{proof}
From Lemma \ref{lem:pairing} and Lemma \ref{lem:escape_saddle_1} we know $\theta$ will converge to some critical point $\theta^*$ such that $v^*=(x,x,\frac{1}{2}-x,\frac{1}{2}-x,0,0)$ for some $x\neq 1/4$ and at the convergence $L_1<\log 2.$ It is worth noting that $L_1<\log 2$ implies $l(r_1)>l(r)$ for all $r\neq r_1$, hence we have $l(r_1)>\sum_{r}p_1(r)l(r)$. We can write the loss $L_1(x,T)$ as a univariate function of $x.$ We can write it out explicitly as 
\[L_1(x)=\log\left(1+\sum_{j=2}^6 \exp\left(d_j(x)/T\right)\right)\] where $d_j(x):=l_{j}(x)-l_1(x)$. It is easy to see $L_1$ is symmetric $L_1(x)=L_1(\frac{1}{4}-x)$. Hence it suffices to show that for small enough $T$, there is only one critical point of $L_1$ in $[0,1/4]$ such that $L_1<\log 2$. Direct caculations give that $d_2=-8x^2+4x-\frac{1}{2}$, $d_3=-3x^2+3x-\frac{3}{4}$, $d_4=-7x^2+4x-\frac{3}{4}$, $d_5=-7x^2+3x-\frac{1}{2}$ and $d_6=-3x^2$. 
The critical point condition is 
\[\Phi(x):=\sum_{j=2}^6 d_j^{\prime}(x) e^{d_j(x) / T}=0.\] We can define the rescaled function $\Psi(x):=e^{-d_6(x)/T}\Phi(x)$ so $\Psi$ and $\Phi$ have same roots. We can rewrite 
\[\Psi(x)=(4-16 x) e^{\left(d_2(x)-d_6(x)\right) / T}-6 x+\sum_{j=3}^5 d_j^{\prime}(x) e^{\left(d_j(x)-d_6(x)\right) / T}.\]
Set $h(x):=d_2(x)-d_6(x)=-5x^2+4x^2-\frac{1}{2}$ and it is easy verify that $h(x)$ is strictly increasing on $(0,1/4)$ and has a unique root 
\[x_0=\frac{4-\sqrt{6}}{10}.\] We can verify that $d_j(x_0)<d_6(x_0)$ for $j=3,4,5$. Hence there exists an surrounding interval $I=(x_0-\epsilon,x_0+\epsilon)\subset(0,1/4)$ such that for all $x\in I$ $d_j(x)-d_6(x)<-\gamma$ for some $\gamma>0$, which implies that 
\[\sum_{j=3}^5d_j'(x)e^{(d_j(x)-d_6(x))/T}=O(\exp(-\gamma/T)).\]
Therefore we can define the dominating term as \[\widetilde \Psi(x):=(4-16 x) e^{\left(d_2(x)-d_6(x)\right) / T}-6 x\]
and we have $|\Psi(x)-\widetilde \Psi(x)|=O(\exp(-\gamma/T))$. It is easy to check the difference of the derivative is also small $|\Psi'(x)-\widetilde\Psi'(x)|=O(\frac{1}{T}\exp(-\gamma/T))$. It is worth noting that these two bounds are uniform for $x\in I$. We can claim that $\widetilde \Psi(x)$ has a unique root in $(0,1/4)$ which has a large derivative.

\begin{claim}
\label{cla:root}
    There is a unique $x_T\in I$ such that  $ \widetilde \Psi(x_T)=0$. Furthermore we have $ \widetilde \Psi'(x_T)=\Theta(\frac{1}{T}).$ 
\end{claim}

With Claim \ref{cla:root}, we can see that $\Psi(x)$ also only have one root in $I$, since with small enough $T$ we have $|\Psi'(x)-\widetilde\Psi'(x)|=o(\widetilde\Psi'(x))$. Outside $I$, if $x<x_0-\epsilon$, then $h(x)<0$ uniformly hence $\Psi(x)<0$. If $x>x_0+\epsilon$ then $h(x)>0$ uniformly and hence $\Psi(x)>0$. Therefore $L_1$ has exactly one critical point $x^*$ in $(0,1/4)$ and it is direct to check $L_1(x^*)<\log 2$ if $T$ is small enough. It remains to prove the Claim.

\begin{proof}[Proof to the Claim \ref{cla:root}]
    We can know $\widetilde\Psi(x)=0$ is equivalent to  \[h(x)=T\log\frac{6x}{4-16x}.\] We can define $F(x):=h(x)-T\log\frac{6x}{4-16x}$ and we have its derivative to be 
    \[F'(x)=h'(x)-T\left(\frac{1}{x}+\frac{16}{4-16x}\right).\]
    We can examine that $h'(x_0)>0$, so for small enough $\epsilon$ we have $F'(x)>0$ in $I$. Also if $T$ is small enough, we have $F(x_0-\epsilon)<0$ since $h(x_0-\epsilon)<0$. Similarly, we have $F(x_0+\epsilon)>0$. Therefore $F$ and  $\widetilde \Psi$ have exactly one root $x_T\in I$. We can calculate to see that \[\widetilde\Psi^{\prime}\left(x_T\right)=-16 e^{h\left(x_T\right) / T}+\frac{6 x_T h^{\prime}\left(x_T\right)}{T}-6 .\]
    The root condition $\widetilde\Psi (x_T)=0$ implies 
    \[\left(4-16 x_T\right) e^{h\left(x_T\right) / T}=6 x_T .\] Plugging it into the derivative of $\widetilde\Psi$ we obtain \[\widetilde{\Psi}^{\prime}\left(x_T\right)=-16 e^{h\left(x_T\right) / T}+\frac{6 x_T h^{\prime}\left(x_T\right)}{T}-6,\]
    where $-16 e^{h\left(x_T\right) / T}=-16\frac{6x_T}{4-16x_T}=O(1)$ and $x_Th'(x_T)=x_T(4-10x_T)=\Theta(1)$. Therefore we have $\widetilde \Psi'(x_T)=\Theta(\frac{1}{T})$.
\end{proof}


\end{proof}

\section{Proof to Theorem \ref{thm:finite_sample}}
\label{sec:proof_finite_sample}
 We first write out the gradient in the general setting. Given an input sequence $Z=(s_1,a_1,s_2,a_2,s_3,u_{\eos})$, we can write its loss as $\ell_1(v;Z,T)$. For any relation $r$, the logit of $r$ is 
 \begin{equation}
 \label{eq:logits_r}
  l(r)=2\sum_{1\leq i\leq 3,1\leq j\leq 2} v_{2i-1}v_{2j}\mathbf{1}_{\{r(s_{i})=a_j\}}+\sum_{i=1}^5v_i^2.   
 \end{equation}
 Therefore the logit of $r$ contains the quadratic term $(v_{2i-1}+v_{2j})^2$ if $r$ maps $s_i$ to $a_j$, which is consistent with Table \ref{tab:logits} in previous three-subject setting. Then the the formulas for gradient w.r.t. $v$ on sequence $Z$ are as follows. 
 

\[
T \frac{\partial \ell_1}{\partial v_1}
=
2\left(
\sum_{r:r(s_1)=a_1} p(r)\, v_2
\;+\;
\sum_{r: r(s_1)=a_2} p(r)\, v_4
\;-\; v_2
\right)
\]

\[
T \frac{\partial \ell_1}{\partial v_2}
=
2\left(
\sum_{r: r(s_1)=a_1} p(r)\, v_1
\;+\;
\sum_{r: r(s_2)=a_1} p(r)\, v_3
\;+\;
\sum_{r:r(s_3)=a_1} p(r)\, v_5 - v_1
\right)
\]

\[
T \frac{\partial \ell_1}{\partial v_3}
=
2\left(
\sum_{r:\, r(s_2)=a_2} p(r)\, v_4
\;+\;
\sum_{r:\, r(s_2)=a_1} p(r)\, v_2 - v_4
\right)
\]

\[
T \frac{\partial \ell_1}{\partial v_4}
=
2\left(
\sum_{r:\, r(s_2)=a_2} p(r)\, v_3
\;+\;
\sum_{r:\, r(s_1)=a_2} p(r)\, v_1
\;+\;
\sum_{r:\, r(s_3)=a_2} p(r)\, v_5 - v_3
\right)
\]

\[
T \frac{\partial \ell_1}{\partial v_5}
=
2\left(
\sum_{r:\, r(s_3)=a_1} p(r)\, v_2
\;+\;
\sum_{r:\, r(s_3)=a_2} p(r)\, v_4
\right)
\]

\[
\frac{\partial \ell_1}{\partial v_6} = 0
\]

We first show that after Stage 1, the transformer will have near perfect accuracy on non-confusing sequences and near $0.5$ accuracy on confusing sequences. This is the area near the saddle point.
\begin{lemma}
    \label{lem: stage_1_finite_sample}
    After Stage 1, if sample size $|D|=w\left(\frac{1}{T^3}\cdot \log^2(\frac{1}{T})\cdot \log(\frac{1}{\delta})\right)$ and $T<\frac{1}{40 \log n}$, with probability at least $1-\delta/3$, the attention scores satisfy that \(\|v - (\tfrac{1}{6}+\alpha,\; \tfrac{1}{6},\; \tfrac{1}{6}+\alpha,\; \tfrac{1}{6},\; \tfrac{1}{6}-2\alpha,\; \tfrac{1}{6})\|_\infty
\lesssim \frac{\eta_1^2}{T^2}, |v_1-v_3|=o(T^2)\) and $|v_2-v_4|=o\left(\exp(-\frac{1}{40T})\right)$ where $\alpha=\frac{c\cdot p_{\mathrm{mis}}\cdot \eta_1}{T}$ for some constant $c$. Moreover, for non-confusing sequences, the accuracy $p_1(r_1)=1-o\left(\exp(-\frac{1}{\sqrt{T}})\right)$; for confusing sequences, the accuracy $p_1(r_1)=\frac{1}{2}-o\left(\exp(-\frac{1}{\sqrt{T}})\right)$.
\begin{proof}
    At initialization, fix an input sequence $Z$. Denote $I(Z)$ the number of $2$-matching relations for $Z$. It is easy to see that all $2$-matching relations have logits $2\cdot (\frac{1}{6}+\frac{1}{6})^2+\frac{1}{6^2}=\frac{1}{4}$, all $1$-matching relations have logits $(\frac{1}{6}+\frac{1}{6})^2+3\cdot \frac{1}{6^2}=\frac{7}{36}$ and all $0$-matching relations have logits $\frac{5}{6^2}=\frac{5}{36}$. Since every two $(s,a)$ pairs correspond to at most one relation, we have the number of $1$-matching relations is at most $6(n-3)$ and the number of $0$-matching relations is at most $n(n-1)-6(n-2)<n^2$.
    Then $1\leq I(Z)\leq 6.$ Note that $$p_1(r_1)=\frac{1}{I(Z)}-O(n\exp(-\frac{1}{18T})+n^2\exp(-\frac{1}{9T})).$$ 
    Therefore if $\frac{1}{T}>40\log n,$ then we have $$p_1(r_1)=\frac{1}{I(Z)}-O\left(\exp(-\frac{1}{36T})\right).$$ 
    
    If $Z$ is not an mismatched sequence, then $$0\leq \frac{\partial \ell_1(Z)}{\partial v_5}-\frac{\partial\ell_1(Z)}{\partial v_i}<O(\exp(-\frac{1}{36T})/T)$$ for all $i\neq 5.$ If $Z$ is a mismatched sequence, then there is at least one $2$-matching relation $r$ that maps $s_3$ either to $a_1$ or to $a_2$, which has prediction probability same as $r_1$ since $r_1$ is also a $2$-matching relation. Hence we have $\frac{\partial\ell_1}{\partial v_5}\geq \frac{p_1(r_1)}{3T}$. Noting that $\frac{\partial\ell_1(Z)}{\partial v_i}\leq 0$ for $i\neq 5$,  we have $$\frac{p_1(r_1)}{3T}\leq \frac{\partial \ell_1(Z)}{\partial v_5}-\frac{\partial\ell_1(Z)}{\partial v_i}<\frac{2}{3T}$$ for all $i\neq 5.$ 
    Summing over $Z$, we have 
    \begin{equation}
    \label{eq:1stGD_v_move}
        \frac{p(r_1)\cdot p_{\mathrm{mis}}}{3T}\leq \frac{\partial L_1(D)}{\partial v_5}-\frac{\partial L_1(D)}{\partial v_i}<\frac{2\cdot p_{\mathrm{mis}}}{3T}+O\left(\exp(-\frac{1}{36T})/T\right).
    \end{equation}
     Also note that (1) $|\frac{\partial \ell_1}{\partial v_i}|\leq O\left(\exp(-\frac{1}{36T})/T\right)$ for $i=2,4,6$; (2) $|\frac{\partial \ell_1}{\partial v_1}+\frac{\partial \ell_1}{\partial v_3}+\frac{\partial \ell_1}{\partial v_5}|=O\left(\exp(-\frac{1}{36T})/T\right).$ Summing over $Z$, we have 
     \begin{equation}
     \label{eq:1stGD_v_2_4}
         \left|\frac{\partial L_1(D)}{\partial v_i}\right|\leq O\left(\exp(-\frac{1}{36T})/T\right) \text{ for } i=2,4,6
     \end{equation}
     and 
     \begin{equation}
     \label{eq:1stGD_v_1_3_5}
         \left|\frac{\partial L_1}{\partial v_1}+\frac{\partial L_1}{\partial v_3}+\frac{\partial L_1}{\partial v_5}\right|=O\left(\exp(-\frac{1}{36T})/T\right).
     \end{equation}
     By (\ref{eq:1stGD_v_2_4}) and (\ref{eq:1stGD_v_1_3_5}), we know at initialization $$|v_0^\top \nabla_vL_1|=O(\exp(-\frac{1}{36})/T).$$ Here $v_0=\left(\frac{1}{6},\frac{1}{6},\frac{1}{6},\frac{1}{6},\frac{1}{6},\frac{1}{6}\right)^\top$ is the attention scores at initialization. Therefore we have $$\|\nabla_\theta L_1\|=O(\|\nabla_vL_1\|)=\Theta(\frac{1}{T}).$$
Since $\frac{\partial L_1}{\partial v_i}=v_i\left(\frac{\partial L_1}{\partial v_i}-v^\top \nabla_v L_1\right)$, we have that $$\|\nabla_\theta L_1(D)-\frac{1}{6}\nabla_v L_1(D)\|=O(\exp(-\frac{1}{36T})/T)$$ at initialization, which also implies that $|v_0^\top \nabla_\theta L_1|=O(\exp(-\frac{1}{36})/T)$. 
     Use Taylor expansion we have after Stage 1 that\[v-v_0=-\eta_1J\nabla_\theta L_1(D)+\epsilon\] where  $J=\mathrm{Diag}(v_0)-v_0v_0^\top$ is the Jacobian of softmax at initialization and the remainder $$\|\epsilon\|=O(\eta_1^2\|\nabla_\theta L_1\|^2)=O(\frac{\eta_1^2}{T^2}).$$ Therefore we can calculate to see that 
     \begin{align*}
         \left|v_1+v_3+v_5-\frac{1}{2}\right|&\lesssim \eta_1\left|\frac{\partial L_1}{\partial v_1}+\frac{\partial L_1}{\partial v_3}+\frac{\partial L_1}{\partial v_5}\right|+\eta_1 |v_0^\top \nabla_\theta L_1|+\|\epsilon\|\\
         &\lesssim \frac{\eta_1^2}{T^2}.
     \end{align*}

      Since we have $$ \frac{\partial L_1(D)}{\partial \theta_i}-\frac{\partial L_1(D)}{\partial \theta_j}=\frac{1}{6}\left(\frac{\partial L_1(D)}{\partial v_i}-\frac{\partial L_1(D)}{\partial v_j}\right)$$ for the gradient at initialization for all $i,j$, applying (\ref{eq:1stGD_v_move}), (\ref{eq:1stGD_v_2_4})  together with Taylor expansion 
      imply that after first step of GD, we also have $$ |v_5-v_i|=\Theta(\frac{p_{\mathrm{mis}\cdot\eta_1}}{T}) \text{ and }  |v_i-v_j|=O\left(\eta_1\cdot\exp(-\frac{1}{36T})/T\right)$$ for $i,j\in\{2,4,6\}$. 
      We still need to show that $v_1$ and $v_3$ are close to each other.
    
    Since the gradients $\left|\frac{\partial \ell_1(Z)}{\partial v_1}\right|, \left|\frac{\partial \ell_1(Z)}{\partial v_3}\right|$ for any sample $Z$ are upper bounded by $\frac{2}{T}$ and $\mathbb{E}_Z[\frac{\partial \ell_1(Z)}{\partial v_1}]=\mathbb{E}_Z[\frac{\partial \ell_1(Z)}{\partial v_3}]$ at initialization. By Chernoff bound, 
    if the sample size $|D|=w\left(\frac{1}{T^3}\cdot \log^2(\frac{1}{T})\cdot \log(\frac{1}{\delta})\right),$ then with probability at least $1-\delta/3,$ we have $$\left|\frac{\partial L_1(D)}{\partial v_1}-\frac{\partial L_1(D)}{\partial v_3}\right|=o\left(\sqrt{T}\log^{-1}(\frac{1}{T})\right),$$ which implies that after Stage 1 we have $$\left|v_1-v_3\right|=\Theta\left(\eta_1 \left|\frac{\partial L_1(D)}{\partial v_1}-\frac{\partial L_1(D)}{\partial v_3}\right|\right)=o(T^2)=o\left(\frac{\eta_1^2}{T^2}\right).$$ Therefore we have \[\|v - (\tfrac{1}{6}+\alpha,\; \tfrac{1}{6},\; \tfrac{1}{6}+\alpha,\; \tfrac{1}{6},\; \tfrac{1}{6}-2\alpha,\; \tfrac{1}{6})\|_\infty
\lesssim \frac{\eta_1^2}{T^2}, |v_1-v_3|=o(T^2)\] and $$|v_2-v_4|=o\left(\exp(-\frac{1}{40T})\right)$$ where $\alpha=\frac{c\cdot p_{\mathrm{mis}}\cdot \eta_1}{T}$ for some constant $c$. With such $v$, it is straightforward to see that for non-confusing sequences, the accuracy $p(r_1)=1-o\left(\exp(-\frac{1}{\sqrt{T}})\right)$; for confusing sequences, the accuracy $p(r_1)=\frac{1}{2}-o\left(\exp(-\frac{1}{\sqrt{T}})\right)$.
\end{proof}
\end{lemma}

We then show the perturbation scheme helps the model escape the saddle by reducing the loss on confusing sequences.
\begin{lemma}
    If $T<T_{\max}(\delta)$ where $T_{\max}(\delta)$ is some constant only depends on $\delta$, with probability at least $1-\delta/3$, after $O(1/T)$ iterations of GD in Stage 2 in Algorithm \ref{alg:temp_perturbed_gd_1}, we have loss on confusing sequences $\ell_1\leq \log 2-\Omega(p_{\mathrm{conf}}^3\cdot T^3).$
\end{lemma}
\begin{proof}
    To apply Lemma 14 in \citet{DBLP:conf/icml/Jin0NKJ17}, we need to examine two properties at $\tilde \theta$: (1) the gradient $\|\nabla L_1(\theta,D)\|$ is small, and (2) the most negative eigenvalue of the Hessian $\lambda_{\min}(\nabla^2L_1(\tilde\theta,D))$ is negative enough. 
    
    We first examine (1). If $Z$ is a non-confusing sequence, it is easy to see $$\|\nabla_v \ell_1(v,Z)\|=o\left(-\exp(-\frac{1}{\sqrt{T}})/T\right).$$ If $Z$ is a confusing sequence, we have $$\|\nabla_v\ell_1(v,Z)\|_\infty \lesssim \frac{1}{T}\max\{|v_1-v_3|, |v_2-v_4|\}.$$ By Lemma \ref{lem: stage_1_finite_sample}, we know $\|\nabla_v\ell_1(v,Z)\|_\infty=o(T)$. Overall we have $$\|\nabla_\theta \ell_1(v, Z)\|=O(\|\nabla_v \ell_1(v, Z)\|)=o(T).$$

    Now we examine (2). Let $\tilde \epsilon=\frac{1}{2}\cdot (1,1,-1,-1,0,0)$ be the potential descending direction of loss. Define $f_{\tilde \epsilon,\tilde \theta}(x,Z):=\ell_1(\tilde \theta+x\cdot\tilde \epsilon,Z)$ to be the loss along the direction of $\tilde \epsilon$ from $\tilde\theta$. Then we have $\lambda_{\min}(\nabla^2 \ell_1(\tilde \theta,Z))\leq f_{\tilde \epsilon,\tilde \theta}^{''}(0,Z).$ We use $p_1(x)$ to denote $p_1(r_1)$ at parameter $\theta=\tilde \theta+x\cdot \tilde \epsilon$. Similarly we use $p_i(x)$ for $p_1(r_i)$ in this proof. Direct calculation gives that 
    \[f_{\tilde \epsilon,\tilde \theta}^{''}(0)=\frac{p_1'(0)^2}{p_1^2(0)}-\frac{p_1''(0)}{p_1(0)}.\] 
    
    We first work with the situation where $Z$ is a confusing sequence. We have
    $$p_1'(0)=\sum_i \frac{\partial p_1}{\partial l_i} \cdot l_i^{\prime}(0)$$ where $\frac{\partial p_1}{\partial l_1}=\frac{1}{T} p_1\left(1-p_1\right)$ and $\frac{\partial p_1}{\partial l_i}=-\frac{1}{T} p_1 p_i$ for $i\neq 1$. For all relations $r$ that are neither the underlying true relation $r_1$ nor the confusing relation $r_2$, the probability of predicting it $p_1(r)=o\left(\exp(-\frac{1}{\sqrt{T}})\right)$ is exponentially small after Stage 1. It is sufficient to only look at $r_1$ and $r_2$.
    Further we have $$l_i^{\prime}(0)=\left(\nabla_v l_i\right)^{\top} v^{\prime}(0)$$ where $\|v'(0)-\delta \odot v\|_\infty = o(T^2)$ since $v_i^{\prime}(0)=\delta_i v_i-v_i v^{\top} \delta$ for each $i$ and $|v^\top \delta|=o(T^2)$ by Lemma \ref{lem: stage_1_finite_sample}. Also we have $$\nabla_v l_1=2\left(v_1+v_2, v_1+v_2, v_3+v_4, v_3+v_4, v_5, 0\right)^{\top}$$ and $$\nabla_v l_2=2\left(v_1+v_4, v_2+v_3, v_2+v_3, v_1+v_4, v_5, 0\right)^{\top}.$$ Hence we have $$l_1'(0)=\left(\nabla_v l_1\right)^{\top} (\delta \odot v) +o(T^2)=(v_1+v_2+v_3+v_4)(v_1-v_3+v_2-v_4)+o(T^2)=o(T^2).$$ Similarly we have $l_2'(0)=o(T^2)$. Hence we obtain that $p_1'(0)=o(T).$ Therefore we have 

    \begin{equation}
        f_{\tilde \epsilon,\tilde \theta}^{''}(0)=-\frac{p_1''(0)}{p_1(0)}+o(T^2).
    \end{equation}

    We have that 
    \begin{equation}
    \label{eq:p_1''}
        p_1^{\prime \prime}(0)=\underbrace{\sum_i \frac{\partial p_1}{\partial l_i} l_i^{\prime \prime}(0)}_{\text {term I }}+\underbrace{\sum_{i, j} \frac{\partial^2 p_1}{\partial l_i \partial l_j} \cdot l_i^{\prime}(0) \cdot l_j^{\prime}(0)}_{\text {term II }}.
    \end{equation}

    First calculate term I. We know $$l_i^{\prime \prime}(0)=v^{\prime}(0)^{\top} \nabla_v^2 l_i(v(0)) v^{\prime}(0)+\left(\nabla_v l_i\right)^{\top} v^{\prime \prime}(0).$$ Note that $$v_i^{\prime \prime}(0)=\delta_i v_i^{\prime}(0)-v_i^{\prime} v^{\top} \delta-v_i (v^{\prime})^\top \delta$$  and $|v^\top\delta|=o(T^2)$, we have $$v_i^{\prime \prime}(0)=v_i\left(\delta_i^2-\sum_{j=1}^6 v_j \delta_j^2\right)+o\left(T^2\right).$$ By Lemma \ref{lem: stage_1_finite_sample}, we have $$\sum_{j=1}^6 v_j \delta_j^2=\frac{1}{6}+O\left(\sqrt{T} \log \frac{1}{T}\right).$$ Hence we have $$v_i^{\prime \prime}(0)=\frac{1}{12} v_i-O\left(\sqrt{T} \log \frac{1}{T}\right),$$ which implies that $$\left(\nabla_v l_i\right)^{\top} v^{\prime \prime}(0)=\frac{1}{24}-O\left(\sqrt{T} \log \frac{1}{T}\right)$$ for $i=1,2$. We also have $$v'(0)^{\top} \nabla_v^2 l_1 v^{\prime}(0)=\frac{1}{12}+O\left(\sqrt{T} \log \frac{1}{T}\right)$$ and $$v'(0)^{\top} \nabla_v^2 l_2 v^{\prime}(0)=O\left(\sqrt{T} \log \frac{1}{T}\right).$$ Therefore we have $l''_1(0)=\frac{1}{8}+O(\sqrt{T}\log \frac{1}{T})$ and $l''_2(0)=\frac{1}{24}-O(\sqrt{T}\log \frac{1}{T})$. So we have
    \begin{align}
    \label{eq:term_I}
        \text{term I}&=\frac{p_1}{T}\left(\frac{1}{8}(1-p_1)-\frac{1}{24}p_2+O\left(\sqrt{T}\log\frac{1}{T}\right)\right)\notag \\&=\frac{p_1}{12T}\left(p_2+O\left(\sqrt{T}\log\frac{1}{T}\right)\right)\notag\\&=\Theta\left(\frac{1}{T}\right).
    \end{align}

For term II, note that $$\frac{\partial^2 p_1}{\partial l_i \partial l_j}=\frac{2}{T^2} p_1 p_i p_j, \frac{\partial^2 p_1}{\partial l_i^2}=\frac{1}{T^2} p_1\left(1-p_1\right)\left(1-2 p_1\right), \frac{\partial^2 p_1}{\partial l_1 \partial l_i}=\frac{p_1 p_i}{T^2}\left(2 p_1-1\right)$$ and $\frac{\partial^2 p_1}{\partial l_i^2}=\frac{p_1 p_i}{T^2}\left(2 p_i-1\right) \text { for } i \neq j, i \neq 1, j \neq 1$ and they are all $o\left(\exp(-\frac{1}{\sqrt{T}})\right)$. Therefore $\text{term II}=o\left(\exp(-\frac{1}{\sqrt{T}})\right)$. Plugging it with (\ref{eq:term_I}) into (\ref{eq:p_1''}), we have $p_1''(0)=-\Theta(1/T)$ and hence $$f_{\tilde \epsilon,\tilde \theta}^{''}(0)\leq -\Omega(1/T)$$ for confusing sequences.

Now we assume that $Z$ is a non-confusing sequence. We then have $|\frac{\partial p_1}{\partial l_i}|=o\left(T^{-1}\exp(-\frac{1}{\sqrt{T}})\right)$ and $|l_i'(0)|=o(T^2)$. Hence $|p_1'(0)|=o\left(T\exp(-\frac{1}{\sqrt{T}})\right)$. So we have for non-confusing $Z$
\[f_{\tilde \epsilon,\tilde \theta}^{''}(0)=-\frac{p_1''(0)}{p_1(0)}+o\left(T^2\exp(-\frac{2}{\sqrt{T}})\right).\]

Similarly we can calculate that for non-confusing $Z$ that
\begin{align}
    \label{eq:term_I}
        \text{term I}&=\frac{p_1}{T}\left(\frac{1}{8}(1-p_1)+O\left(\sqrt{T}\log\frac{1}{T}\right)\right)\notag \\&=O\left(T^{-\frac{1}{2}}\log(\frac{1}{T})\right)\notag\\
        &=o\left(\frac{1}{T}\right)
    \end{align}
    and $\text{term II}=o\left(\exp(-\frac{1}{\sqrt{T}})\right)$. Hence for non-confusing $Z$ we have $p_1''(0)=o(\frac{1}{T})$ and $|f_{\tilde \epsilon,\tilde \theta}^{''}(0,Z)|=o(\frac{1}{T})$. Summing over all confusing and non-confusing $Z$, we have $f_{\tilde \epsilon,\tilde \theta}^{''}(0,D)\lesssim-\Omega(p_{\mathrm{conf}}/T)$ and therefore $\lambda_{\min} (\nabla^2 L_1(\tilde \theta,D))\lesssim -\Omega(p_{\mathrm{conf}}/T)$.

Now we can apply Lemma 14 in \citet{DBLP:conf/icml/Jin0NKJ17} and obtain that after $\tilde O(\frac{1}{p_{\mathrm{conf}}\cdot T})$ iterations of GD in Stage 2, with probability at least $1-\delta/3$, we have $$L_1(\theta,D)\leq L_1(\tilde \theta,D)-\Omega (p_{\mathrm{conf}}^3\cdot T^3).$$ We know for non-confusing sequences we have $\ell_1(\tilde \theta, Z)=o\left(\exp(-\frac{1}{\sqrt{T}})\right),$ hence denoting $D_{\text{conf}}\subset D$ the set of confusing sequences in $D$,  we must have $$L_1(\theta,D_{\text{conf}})\leq L_1(\tilde \theta,D_{\text{conf}})-\Omega (p_{\mathrm{conf}}^3\cdot T^3)=\log 2-\Omega (p_{\mathrm{conf}}^3\cdot T^3).$$ For any confusing sequence $Z$, defining $g:=l(r_1)-l(r_2)=2(v_1-v_3)(v_2-v_4)$, it is straightforward to verify that 
\[\frac{1}{1+\exp(-g/T)+4\exp(-\sqrt{\frac{1}{T}})+n^2\exp(-\frac{1}{18T})}\leq p_1(r_1,Z)\leq\frac{1}{1+\exp(-g/T)}.\] Therefore for any two confusing sequences $Z_1$ and $Z_2$, we know $$\left|p_1(r_1,Z_1)-p_1(r_1,Z_2)\right|=O(\exp(-\frac{1}{\sqrt{T}}))$$ is exponentially small, which implies that $|\ell_1(\theta,Z_1)-\ell_1(\theta,Z_2)|=O(\exp(-\frac{1}{\sqrt{T}}))$. Combining it with that $L_1(\theta,D_{\text{conf}})\leq \log 2-\Omega (p_{\mathrm{conf}}^3\cdot T^3)$, we know $$\ell_1(\theta,Z)\leq \log2-\Omega(p_{\mathrm{conf}}^3\cdot T^3)$$ for any confusing sequence $Z$.

\end{proof}

Finally we show that after escaping the saddle, the loss on confusing sequences can decrease fast.
\begin{lemma}
\label{lem:stage_3_finite_sample}
    After $O(\frac{T}{\eta_2\cdot p_{\mathrm{conf}}}\log\frac{1}{T})$ iterations of GD in Stage 2 in Algorithm \ref{alg:temp_perturbed_gd_1}, with probability at least $1-\delta/3,$ we have $p(r_1)\geq 0.999$ on any confusing sequence.
\end{lemma}

\begin{proof}
    Assume $Z\in D$ is a confusing sequence. We denote $p_2(Z)$ the prediction probability for $r_2$ on $Z$ and $p_2(D):=\frac{1}{p_{\text{conf}}\cdot|D|}\sum_{\text{confusing~} Z'\in D}p_2(Z')$.  After \(\tilde O(1/T)\) iterations of GD in Stage 2, with probability least $1-\delta/3$, we have \[\ell_1(\theta, Z) \leq \log 2 - \Omega(p_{\mathrm{conf}}^3\cdot T^3).\]Hence \(p(r_2) \le \tfrac12 - \Omega(T^3)\), which implies that \[(v_1-v_3)(v_2-v_4)\gtrsim T^4\] since $l_1-l_2=2(v_1-v_3)(v_2-v_4)$. Without loss of generality, assume \(v_1>v_3\) and \(v_2>v_4\). Let
    \[\Delta:=\mathrm{softmax}(\theta-\eta_2\nabla_\theta L_1(\theta))-\mathrm{softmax}(\theta)\in\mathbb{R}^6\] be the change of the attention scores after one step of gradient descent. Denote \[g:=(v_1-v_3)(v_2-v_4).\] Then we have
    \begin{equation}
    \label{eq:finite_g_growth}
        g_{k+1}-g_k=(v_{1}-v_{3})(\Delta_2-\Delta_4)+(v_2-v_4)(\Delta_1-\Delta_3)+(\Delta_1-\Delta_3)(\Delta_2-\Delta_4)
    \end{equation} where all $\Delta$ and $v_i$ at RHS are at time $k$. We also have $g_0\gtrsim T^4$.

Now we show that \(g_{k+1}-g_k\geq \frac{\eta_2 p_{\mathrm{conf}}}{10000T}\cdot g_k\) if \(p(r_1)<0.999\).

As long as \(k\leq \frac1T(\log \tfrac1T)^{1.5}\), by Lemma \ref{lem:movement_upper_bound} we have
\[
\|\theta(k)-\tilde\theta\|_2\lesssim \sqrt{k\cdot T^2}\lesssim T^{1/2}(\log \tfrac1T)^{3/4},
\]
which implies \(\|v(k)-\tilde v\|_2\lesssim T^{1/2}(\log \tfrac1T)^{3/4}\) and \(p(r_i)=o(\exp^{-1/\sqrt{T}})\) for $i\neq 1,2$.

Also we have \[\frac{\partial \ell_1(v,Z)}{\partial v_1}=\frac1T\bigl(p_2(Z)(v_4-v_2)+o(e^{-1/\sqrt{T}})\bigr),\] \[\frac{\partial \ell_1(v,Z)}{\partial v_2}=\frac1T\bigl(p_2(Z)(v_3-v_1)+o(e^{-1/\sqrt{T}})\bigr),\] \[\frac{\partial \ell_1(v,Z)}{\partial v_3}=\frac1T\bigl(p_2(Z)(v_2-v_4)+o(e^{-1/\sqrt{T}})\bigr),\] \[\frac{\partial \ell_1(v,Z)}{\partial v_4}=\frac1T\bigl(p_2(Z)(v_1-v_3)+o(e^{-1/\sqrt{T}})\bigr),\] \[\frac{\partial \ell_{1}(v,Z)}{\partial v_5}=o(e^{-1/\sqrt{T}})\] and \[\frac{\partial \ell_1(v,Z)}{\partial v_6}=0.\] We also know $$\|\nabla_v \ell_1(v,Z')\|_{\infty}=o\left(T^{-1}\cdot\exp(-1/\sqrt{T})\right)$$ for non-confusing sequences $Z'$. Therefore $$\frac{\partial L_1}{\partial v_i}=p_{\mathrm{conf}}\cdot \frac{\partial \ell_1(v,Z)}{\partial v_i}+o\left(T^{-1}\cdot\exp(-1/\sqrt{T})\right)$$ for all $i$. Moreover, \[\frac{\partial L_1}{\partial \theta_i}=v_i\left(\frac{\partial L_1}{\partial v_i}-\sum_{j=1}^6 v_j\frac{\partial L_1}{\partial v_j}\right)\] and \[\sum_{j=1}^6 v_j\frac{\partial L_1}{\partial v_j}=\frac{4p_2\cdot p_{\mathrm{conf}}}{T}\bigl(-g+o(e^{-1/\sqrt{T}})\bigr).\]

Now we are ready to calculate $\Delta_1-\Delta_3$ and $\Delta_2-\Delta_4$ in (\ref{eq:finite_g_growth}).

By Taylor expansion we have 
    \[\Delta=-\eta_2 J(v(\theta))\nabla_\theta L_1+\epsilon\]
    where $J(v(\theta))=\mathrm{Diag}(v)-vv^\top$ is the Jacobian of softmax and $\|\epsilon\|=O(\eta_2^2\|\nabla_\theta L_1(\theta)\|^2)$.

 Therefore we have \[\Delta_1-\Delta_3=\eta_2\left((v_1-v_3)v^\top\nabla_\theta L_1+v_3\frac{\partial L_1}{\partial \theta_3}-v_1\frac{\partial L_1}{\partial \theta_1}\right)+\epsilon_1-\epsilon_3\]
 and 
 \[\Delta_2-\Delta_4=\eta_2\left((v_2-v_4)v^\top\nabla_\theta L_1+v_4\frac{\partial L_1}{\partial \theta_4}-v_2\frac{\partial L_1}{\partial \theta_2}\right)+\epsilon_2-\epsilon_4.\]
Plug them into \ref{eq:finite_g_growth} we have 

\begin{align}
\label{eq:g_growth_2}
g_{k+1} - g_k
&= \eta_2 \Bigg[
2 (v_1 - v_3)(v_2 - v_4)\, v^\top \nabla_\theta L_1 + \eta_2 (v_1 - v_3)(v_2 - v_4)\bigl(v^\top \nabla_\theta L_1\bigr)^2 \notag \\
&\quad + (v_2 - v_4)\left(v_3 \frac{\partial L_1}{\partial \theta_3} - v_1 \frac{\partial L_1}{\partial \theta_1}\right) + (v_1 - v_3)\left(v_4 \frac{\partial L_1}{\partial \theta_4} - v_2 \frac{\partial L_1}{\partial \theta_2}\right) \notag \\
&\quad + \eta_2 \left(v_3 \frac{\partial L_1}{\partial \theta_3} - v_1 \frac{\partial L_1}{\partial \theta_1}\right)
\left(v_4 \frac{\partial L_1}{\partial \theta_4} - v_2 \frac{\partial L_1}{\partial \theta_2}\right)
\notag \\
&\quad + \eta_2(v_1-v_3)v^\top \nabla_\theta L_1\left(v_4\frac{\partial L_1}{\partial \theta_4}-v_2\frac{\partial L_1}{\partial \theta_2}\right)\notag\\
&\quad + \eta_2(v_2-v_4)v^\top \nabla_\theta L_1\left(v_3\frac{\partial L_1}{\partial \theta_3}-v_1\frac{\partial L_1}{\partial \theta_1}\right) \Bigg] +O(\|\epsilon\|)
\end{align}

We first obtain 
\begin{align}
\label{eq:finite_inter-step-1}
     v_3\frac{\partial L_1}{\partial \theta_3}-v_1\frac{\partial L_1}{\partial \theta_1}&=v_3^2 \frac{\partial L_1}{\partial v_3}-v_1^2 \frac{\partial L_1}{\partial v_1}+\left(v_1^2-v_3^2\right) \cdot v^{\top} \nabla_v L_1 \notag\\
     &=\frac{2 p_2(D)p_{\mathrm{conf}}}{T}\left(\left(v_1^2+v_3^2\right)\left(v_2-v_4\right)+2\left(v_3^2-v_1^2\right) g+o\left(e^{-\frac{1}{\sqrt{T}}}/T\right)\right).
\end{align}

Similarly we have 
\begin{equation}
\label{eq:finite_inter-step-2}
    v_4\frac{\partial L_1}{\partial \theta_4}-v_2\frac{\partial L_1}{\partial \theta_2}= \frac{2 p_2(D)p_{\mathrm{conf}}}{T}\left(\left(v_2^2+v_4^2\right)\left(v_1-v_3\right)+2\left(v_4^2-v_2^2\right) g+o\left(e^{-\frac{1}{\sqrt{T}}}/T\right)\right).
\end{equation}

Also we have 
\begin{align}
    v^\top \nabla_\theta L_1&=\sum_{i=1}^6 v_i^2\frac{\partial L_1}{\partial v_i}-\|v\|^2 v^\top \nabla_v L_1 \notag \\
    &=\frac{2p_2(D)p_{\mathrm{conf}}}{T}\left((v_3^2-v_1^2)(v_2-v_4)+(v_4^2-v_2^2)(v_1-v_3)+2\|v\|^2g+o\left(e^{-1/\sqrt{T}}\right)\right)\notag\\
    &=\frac{2p_2(D)p_{\mathrm{conf}}}{T}\left(\left(2\|v\|^2-(v_1+v_2+v_3+v_4)\right)g+o\left(e^{-1/\sqrt{T}}\right)\right),\notag
\end{align}
which implies that 
\begin{equation}
\label{eq:finite_inter-step-3}
    \left|v^\top \nabla_\theta L_1\right|=O\left(\frac{p_2(D)p_{\mathrm{conf}}g}{T}\right).
\end{equation}

Plugging (\ref{eq:finite_inter-step-1}), (\ref{eq:finite_inter-step-2}) and (\ref{eq:finite_inter-step-3}) into (\ref{eq:g_growth_2}) and assuming that $g=o(1)$ (if $g=\Omega(1)$ then it is easy to see $p_1\geq 0.999$), we finally obtain that 

\begin{align}
    g_{k+1}-g_k&= \frac{2 p_2(D)p_{\mathrm{conf}}\eta_2}{T} \left( (v_1^2 + v_3^2)(v_2 - v_4)^2 + (v_2^2 + v_4^2)(v_1 - v_3)^2 + o(g) \right)+O(\|\epsilon\|) \notag\\
    &\geq \frac{2p_2(D)p_{\mathrm{conf}}\eta_2}{T} \left(\sqrt{(v_1^2+v_3^2)(v_2^2+v_4^2)}g+o(g)\right)\notag \\
    &\geq \frac{\eta_2p_{\mathrm{conf}}}{10000T}\cdot g
\end{align}
for sufficiently small $T$ and $p_2(D)>0.0009$ (If $p_2(D)\leq 0.0009$ then we have $p_1(D)\geq 0.9991$). The first inequality uses the fact that 
\begin{align*}
  \|\epsilon\|&\lesssim\eta_2^2\|\nabla_\theta L_1\|^2\\ &\lesssim\eta_2^2\|\nabla_vL_1\|^2\\
&\lesssim\frac{p_2^2p_{\mathrm{conf}}^2\eta_2^2}{T^2}\left((v_1-v_3)^2+(v_2-v_4)^2\right)\\
&=o\left(\frac{p_2p_{\mathrm{conf}}\eta_2}{T}\left((v_1^2 + v_3^2)(v_2 - v_4)^2 + (v_2^2 + v_4^2)(v_1 - v_3)^2\right)\right)  
\end{align*}
 and $(v_1^2 + v_3^2)(v_2 - v_4)^2 + (v_2^2 + v_4^2)(v_1 - v_3)^2\geq 2\sqrt{(v_1^2+v_3^2)(v_2^2+v_4^2)}g$. 

Hence \(g_{k+1}-g_k\ge \frac{\eta_2p_{\mathrm{conf}}}{10000T}g_k\) as long as \(p_1(D)<0.9991\). This means \(g_{k+1}\ge \left(1+\frac{\eta_2p_{\mathrm{conf}}}{10000T}\right)g_k\). Since we have $g_0\gtrsim T^4$, there exists constant $C_2$, such that when \(k\ge  \frac{C_2T\log(1/T)}{\eta_2p_{\mathrm{conf}}}\), we have $p_1(D)\geq 0.9991$. It is straightforward to verify that at time $k$ we have 
\[\frac{1}{1+\exp(-g_k/T)+4\exp(-\sqrt{\frac{1}{T}})+n^2\exp(-\frac{1}{18T})}\leq p_1(Z')\leq\frac{1}{1+\exp(-g_k/T)}\]
 for any confusing sequence $Z'$. Therefore for $T<\frac{1}{40\log n}$ and large enough $n$ we have $|p_1(Z')-p_1(Z)|<0.0001$, which implies $p(r_1)\geq 0.999$ for any confusing sequence.
\end{proof}

Lemma \ref{lem:stage_3_finite_sample} shows that the test accuracy is high on confusing sequences. To prove Theorem \ref{thm:finite_sample}, it remains to show that the accuracy is also high on non-confusing sequences.
\begin{proof}[Proof of Theorem \ref{thm:finite_sample}]
    We first show that after Stage 2, $p(r_1)>0.99$ for any non-confusing sequence as well. We know that after Stage 2, $\|v(\theta)-v(\tilde \theta)\|\lesssim \sqrt{T}(\log\frac{1}{T})^{\frac{3}{4}}=o(\alpha)=o(\frac{p_{\mathrm{mis}\cdot \eta_1}}{T})$. Hence by Lemma \ref{lem: stage_1_finite_sample}, we have $p(r_1)>0.999$ for non-confusing sequences if $p_{\mathrm{mis}}$ is a constant. By Lemma \ref{lem:stage_3_finite_sample}, we also need to show that $p_{\mathrm{conf}}$ cannot be too small. Apply Chernoff bound to Assumption \ref{assum:sequence_distribution}, we have that if $|D|\geq \frac{13}{\zeta}\log\frac{1}{\delta}$, then with probability at least $1-\delta/3$, we have $p_{\mathrm{conf}}, p_{\mathrm{mis}}\geq \zeta/2$. Taking a union bound over the randomness here and the randomness in Lemma \ref{lem: stage_1_finite_sample}, Lemma \ref{lem:stage_3_finite_sample}, we finish the proof.
\end{proof}

\section{Proof to Lemma \ref{lem:2nd_decoding}}
\label{sec:proof_2nd_decoding}
In this section, we denote $q=\mathrm{softmax}(\omega)$ the attention scores from the first decoding token. We abbreviate $p_{2,r}(a,Z)$ as $p_{2,r}(a)$ if $Z$ can be inferred from the context. We restate Lemma \ref{lem:2nd_decoding} here.


\secondDecoding*

\begin{proof}
     We fix any sampled $\widetilde Z=(s_1,a_1,s_2,a_2,s_3,u_{\mathrm{EoS}})$. The prediction logits for the second decoding step conditioned on that the first decoded token being $r_1$ are
    \[l(a_1)=(q_1+(q_7+1))^2+q_3^2+q_5^2,\]
    \[l(a_2)=(q_3+(q_7+1))^2+q_1^2+q_5^2,\]
    \[l(a_3)=(q_5+(q_7+1))^2+q_3^2+q_1^2\]
    and $l(a)= q_1^2+q_3^2+q_5^2+(q_7+1)^2$ for any answer $a\notin \{a_1,a_2,a_3\}$. We see that the loss for the second decoding step as a function of $q$ is identical for every sequence, and hence it suffices to analyze a single sequence. We can also see that the gap between $l(a_3)$ and the logits of other answers is lower bounded by $2(q_5-\max\{q_1,q_3\})$. Our goal is to show this lower bound of the gap becomes large enough after Stage 1, and remains large throughout Stage 2.

    We first calculate the gradient of the loss $\ell_2(\omega,\widetilde Z)$ w.r.t. attention scores $q_1,q_3$ and $q_5$ in the following claim.
    \begin{claim}
    \label{cla:gradients}
    The gradients w.r.t. $q$ are
        \[\frac{\partial \ell_2}{\partial q_1}=\frac{2(1+q_7)}{T}p_{2,r_1}(a_1),\]

    \[\frac{\partial \ell_2}{\partial q_3}=\frac{2(1+q_7)}{T}p_{2,r_1}(a_2),\]

    \[\frac{\partial \ell_2}{\partial q_5}=\frac{2(1+q_7)}{T}\left(p_{2,r_1}(a_3)-1\right),\]
\[\frac{\partial\ell_2}{\partial q_7}=\frac{2}{T}\left(p_{2,r_1}(a_1)q_1+p_{2,r_1}(a_2)q_3+p_{2,r_1}(a_3)q_5-q_5\right),\]
    \[\frac{\partial\ell_2}{\partial q_i}=0 \text{~~for~}i\neq1,3,5,7.\]
    \end{claim}



    We have $\frac{\partial\ell_2}{\partial \omega_i}=q_i(\frac{\partial\ell_2}{\partial q_i}-\sum_{j=1}^7q_j\frac{\partial\ell_2}{\partial q_j})$ and at initialization $q_i=\frac{1}{7}$ for all $i\in [7]$, $p_{2,r_1}(a_1)=p_{2,r_1}(a_2)=p_{2,r_1}(a_3)$. Therefore at initialization we have 
    \begin{align*}
        \nabla_{q}\ell_2 = \frac{1}{7T}\begin{pmatrix}
            16p \\ 0 \\ 16p \\ 0 \\ 16(p - 1) \\ 0 \\ 2(3p-1)
        \end{pmatrix} \Longrightarrow \nabla_{\omega}\ell_2 = \frac{1}{343T}\begin{pmatrix}
            58p + 18 \\ -54p+18 \\ 58p + 18 \\ -54p+18 \\ 58p -94 \\ -54p+18 \\ -12p + 4
        \end{pmatrix}
    \end{align*}
    where $p := p_{2, r_1}(a_1) < \frac13$. Therefore after the first GD step, we have
    \begin{align*}
        \tilde\omega_5 - \max_{i \neq 5} \tilde\omega_i \ge \frac{32\eta_1}{147T} = \Theta(\sqrt{T}\log(1/T)).
    \end{align*}
    Since the perturbation is of radius $\Theta(T^3\log^{-2}(1/\delta))$, for sufficeintly small $T$, after the perturbation the parameter $\omega$ continues to satisfy $\omega_5 - \max_{i \neq 5} \omega_i = \Theta(\sqrt{T}\log(1/T)).$

    Next, we show that if $\omega_5 - \max_{i \neq 5} \omega_i > 0$, then after a GD step this quantity is non-decreasing. It is immediate to see $\frac{\partial \ell_2}{\partial q_i} - \frac{\partial \ell_2}{\partial q_5} \ge 0$ for $i = 1, 2, 3, 4, 6$. Next, since $\omega_5 \ge \omega_i$, $q_5 \ge q_i$ and thus
    \begin{align*}
        \frac{\partial \ell_2}{\partial q_7} \ge \frac{2q_5}{T}(p_{2, r_1}(a_3) - 1) \ge \frac{2(1 + q_7)}{T}(p_{2, r_1}(a_3) - 1) = \frac{\partial \ell_2}{\partial q_5}
    \end{align*}
    as well. Finally,
    \begin{align*}
        \frac{\partial \ell_2}{\partial \omega_i} - \frac{\partial \ell_2}{\partial \omega_5} &= q_i\frac{\partial \ell_2}{\partial q_i} - q_5\frac{\partial \ell_2}{\partial q_5} + (q_5 - q_i)\sum_{j=1}^7 q_j \frac{\partial \ell_2}{\partial q_j}\\
        &\ge q_i\frac{\partial \ell_2}{\partial q_i} - q_5\frac{\partial \ell_2}{\partial q_5} + (q_5 - q_i)\frac{\partial \ell_2}{\partial q_5}\\
        &= q_i\left(\frac{\partial \ell_2}{\partial q_i} - \frac{\partial \ell_2}{\partial q_5} \right)\\
        &\ge 0.
    \end{align*}
    Hence $\omega_5 - \max_{i \neq 5} \omega_i$ is non-decreasing. Therefore at any point during Stage 2, we have
    \begin{align*}
        q_5-q_i = q_5(1 - \frac{q_i}{q_5}) = q_5(1 - \exp(\omega_i - \omega_5)) \ge \Theta(\sqrt{T}\log(1/T))  
    \end{align*} 
    for sufficiently small $T$ since $q_5 \ge 1/7$. This implies the logit gap $l(a_3)-\max\{l(a_1),l(a_2)\}=\Omega(\sqrt{T}\log\frac{1}{T})$, as well as $l(a_3) - l(a) \ge 2q_5 = \Omega(1)$ for any $a \not\in \{a_1, a_2, a_3\}$. Therefore for $\frac{1}{T}>40\log n$, we have   
        \[p_{2,r_1}(a_3)=\frac{1}{1+2\exp\left(-\Omega\left(\frac{1}{\sqrt{T}}\log\frac{1}{T}\right)\right)+(n-3)\exp\left(-\Omega\left(\frac{1}{T}\right)\right)}\geq 1-o\left(\exp\left(-\frac{1}{\sqrt{T}}\right)\right).\]

\begin{proof}[Proof to Claim \ref{cla:gradients}]
        Conditioned on the correct decoded relation $r_1$, we have the  loss for the second decoding step $\ell_2=-\log\left(p_{2,r_1}(a_3)\right)$ where $p_{2,r_1}(a_3)=\frac{\exp(l(a_3)/T)}{\sum_{a\in\mathcal{A}}\exp(l(a)/T)}$. Therefore we have the gradient w.r.t. the logit vector $l$ is
        \[\frac{\partial \ell_2}{\partial l}=\frac{1}{T}(p_{2,r_1}-e_3).\]

        Also we know that \[l(a_1)=(q_1+(q_7+1))^2+q_3^2+q_5^2,\]
    \[l(a_2)=(q_3+(q_7+1))^2+q_1^2+q_5^2,\]
    \[l(a_3)=(q_5+(q_7+1))^2+q_3^2+q_1^2\]
    and $l(a)= q_1^2+q_3^2+q_5^2+(q_7+1)^2$ for any answer $a\notin \{a_1,a_2,a_3\}$. Therefore we can obtain $\frac{\partial l_i}{\partial q_j}=0$ if $j\neq 1,3,5,7$,
        \[\frac{\partial l_i}{\partial q_{2j-1}}=2q_{2j-1}+2(q_7+1)\delta_{ij}  \text{ for any } i\in[n], j=1,2,3\]

        and \[\frac{\partial l_i}{\partial q_7}=2(q_7+1)+2q_{2i-1}\cdot \mathbf{1}_{\{i\leq 3\}} ~\text{for any }i\in[n].\]
        Since $\sum_{i}p_{2,r_1}(a_i)=1$, we have
        \[\frac{\partial \ell_2}{\partial q_{2j-1}}=\sum_{i\in [n]}\frac{\partial\ell_2}{\partial l_i}\frac{\partial l_i}{\partial q_{2j-1}}=\frac{2(q_7+1)}{T}\left(p_{2,r_1}(r_1(s_{j}))-\delta_{j3}\right)\]
and \[\frac{\partial\ell_2}{\partial q_7}=\frac{2}{T}(p_{2,r_1}(r_1(s_1))q_{1}+p_{2,r_1}(r_1(s_2))q_{3}+p_{2,r_1}(r_1(s_3))q_{5}-q_5).\] 
        Plugging into the mapping of $r_1$, we obtain the desired results. For $j\neq 1,3,5,7,$ we have $\frac{\partial\ell_2}{\partial q_j}=0.$
    \end{proof}
\end{proof}

\section{Helper lemmas}
\label{sec:helper_lemmas}
\begin{proposition}
    There exists constant $T_0$ such that for any $0<T<T_0$, the loss $L_1(\theta, T)$ has gradient Lipschitz constant $O(1/T^2)$ and Hessian Lipschitz constant $O(1/T^3).$
\end{proposition}
\begin{proof}
We prove the proposition for the general setting where $|\mathcal{S}|=n$ and $|\mathcal{R}|=m$. Recalling that loss $L_1(\theta, T)=\mathbb{E}_Z[\ell_1(\theta,Z,T)]$, it suffices to show for $\ell_1(\theta, Z,T)$ that the desired property holds. Also note that since $v=\mathrm{softmax}(\theta)$, its derivatives w.r.t. $\theta$ up to order $3$ are uniformly bounded by absolute constants. Therefore by chain rule, it suffices to show for $\ell_1(v, Z,T)$ that its gradient and Hessian Lipschitz constants are $O(1/T^2)$ and $O(1/T^3)$ respectively for any $Z$. We fix any $Z$. Then by (\ref{eq:logits_r}) we know each logit $l(r_j)$ can be written as some quadratic form 
\[l(r_j)=v^\top A_jv\]
for some symmetric matrix $A_j$, and the number of possible quadratic forms is at most $13$, which is independent of $n$. For simplicity, we abbreviate $l(r_j)$ as $l_j$. Then on the simplex of $v$, there exists absolute constants $G,H>0$ such that
\[\|\nabla_v l_j(v)\|\leq G,\quad \|\nabla^2_v l_j(v) \|\leq H,\quad \nabla^3_vl_j(v)=0 ~\text{ for all }j\in[m].\]

 We know that the loss $$\ell_1(v,T)=-\log p_1(v)=-\frac{l_1(v)}{T}+\log\sum_{j=1}^m\exp\left(\frac{l_j(v)}{T}\right).$$
Define the log-sum-exp term as $\Psi(v):=\log\sum_{j=1}^m\exp\left(\frac{l_j(v)}{T}\right).$ For a $k$-th order tensor $A\in\mathbb{R}^{d\times\dots\times d}$, we define the associated $k$-linear form as $$A[x_1,\dots,x_k]:=\sum_{i_1,\dots,i_k=1}^dA_{i_1,\dots,i_k}\cdot(x_{1})_{i_1}\cdots (x_{k})_{i_k}, \quad x_1,\dots,x_k\in\mathbb{R}^d.$$
For unit vectors $a,b$ we define
\[X^a_j:=\nabla_v l_j[a], \quad Y_j^{a,b}:=\nabla^2l_j[a,b].\]
Then we have $|X_j^a|\leq G$ and $|Y^{a,b}_j|\leq H.$ We first derive the upper bound for the gradient.

\textbf{Gradient Lipschitz constant.}  for all $j\in[m]$. Differentiating $\Psi$ gives \[\nabla \Psi(v)[a]=\frac{1}{T}\sum_{j\in [m]} p_jX^a_j\]

and \[\nabla^2\Psi(v)[a,b]=\frac{1}{T}\sum_{j\in[m]}p_jY^{a,b}_j+\frac{1}{T^2}\mathrm{Cov}_p(X^a,X^b).\]
Here $p=\mathrm{softmax}(l)$ is the prediction probability vector and $$\mathrm{Cov}_p(X,Y):=\sum_{j=1}^mp_jX_jY_j-\left(\sum_{j=1}^mp_jX_j\right)\cdot \left(\sum_{j=1}^mp_jY_j\right) \text{ for any } X,Y\in\mathbb{R}^m.$$
It is easy to see that \[\left|\nabla\Psi(v)[a]\right|\leq \frac{G}{T}=O(1/T)\] and there exists $T_0>0$ such that for any $0<T<T_0$ we have \[|\nabla^2\Psi(v)[a,b]|\leq \frac{H}{T}+\frac{2G^2}{T^2}=O(1/T^2).\] Also noting that \(|\nabla^2 l_1[a,b]|\leq H\),
therefore we have \[|\nabla^2\ell_1[a,b]|=O(1/T^2),\]
which implies that 
\[\|\nabla^2\ell_1(v)\|=O(1/T^2).\]

\textbf{Hessian Lipschitz constant.}  We can calculate to obtain that for unit vectors $a,b,c$, we have
\begin{align*}
    \nabla^3 \Psi(v)[a,b,c]&=\frac{1}{T}\sum_{j\in[m]}p_j\nabla^3l_j[a,b,c]\\&+\frac{1}{T^2}\left(\mathrm{Cov}_p(Y^{a,b},X^c)+\mathrm{Cov}_p(Y^{a,c},X^b)+\mathrm{Cov}_p(Y^{b,c},X^a)\right)\\ &+\frac{1}{T^3}\sum_{j=1}^m p_j\left(X_j^a-\mu_a\right)\left(X_j^b-\mu_b\right)\left(X_j^c-\mu_c\right)
\end{align*}
where $\mu_a:=\sum_{j\in [m]}p_jX^a_j$ and $\mu_b$, $\mu_c$ similarly follow. Using $\nabla^3 l_j=0$, $|X^a_j|, |X^b_j|, |X^c_j|\leq G$ and $|Y^{a,b}|, |Y^{b,c}|, |Y^{a,c}|\leq H$ for all $j\in[m]$, we obtain \[|\nabla^3\Psi(v)[a,b,c]|\leq \frac{6HG}{T^2}+\frac{8G^3}{T^3}=O(1/T^3).\] Therefore combining it with $\nabla^3 l_1=0$ we have \[\left|\nabla^3\ell_1[a,b,c]\right|=O(1/T^3),\]
 which implies that \[\|\nabla^3\ell_1(v)\|=O(1/T^3).\]

\end{proof}

\begin{lemma}[GD movement upper bound]
\label{lem:movement_upper_bound}
Let $f(\cdot)$ be $\ell$-gradient Lipschitz smooth and $f^\star = \inf_{x} f(x)$. 
Consider gradient descent iterates $x_{k+1} = x_k 
- \eta\nabla f(x_k)$ where $\eta\leq\frac{1}{\ell}$. Then for any $t \geq 1$,
\[
\|x_t - x_0\| 
\leq \sqrt{\frac{2t\,(f(x_0) - f^\star)}{\ell}}\,.
\]
\end{lemma}

\begin{proof}
By the descent lemma, $\sum_{k=0}^{t-1} \|x_{k+1} - x_k\|^2 
= \eta^2\sum_{k=0}^{t-1}\|\nabla f(x_k)\|^2 
\leq \frac{2(f(x_0) - f^\star)}{\ell}$. 
Applying the triangle inequality followed by Cauchy--Schwarz,
\[
\|x_t - x_0\| 
\leq \sum_{k=0}^{t-1}\|x_{k+1} - x_k\|
\leq \sqrt{t\sum_{k=0}^{t-1}\|x_{k+1} 
- x_k\|^2}
\leq \sqrt{\frac{2t\,(f(x_0) - f^\star)}{\ell}}.
\]
\end{proof}



\end{document}